\documentclass[11pt,letterpaper]{article}
\usepackage[margin=1in]{geometry}
\usepackage{amsmath,amssymb,amsthm,bm}
\usepackage{graphicx}
\usepackage{booktabs}
\usepackage{natbib}
\usepackage{setspace}
\usepackage[font={small,stretch=1}]{caption}
\usepackage{etoolbox}
\usepackage{hyperref}
\usepackage{enumitem}
\usepackage{algorithm}
\usepackage{algpseudocode}
\usepackage{bbm}
\usepackage[T1]{fontenc}
\usepackage[utf8]{inputenc}
\usepackage{lmodern}
\IfFileExists{microtype.sty}{\usepackage{microtype}}{}
\hypersetup{
  pdftitle={Conformal Prediction for Spatially Dependent Data via Sequential Whitening},
  colorlinks=true,
  linkcolor=blue,
  citecolor=blue,
  urlcolor=blue
}
\theoremstyle{plain}
\newtheorem{theorem}{Theorem}
\newtheorem{proposition}{Proposition}

\newtheorem{lemma}{Lemma}
\theoremstyle{definition}
\newtheorem{assumption}{Assumption}

\theoremstyle{remark}

\newcommand{\R}{\mathbb{R}}
\newcommand{\E}{\mathbb{E}}
\newcommand{\Prob}{\mathbb{P}}
\newcommand{\Var}{\mathrm{Var}}
\newcommand{\Cov}{\mathrm{Cov}}
\newcommand{\Sig}{\bm{\Sigma}}

\newcommand{\bSigma}{\bm{\Sigma}}
\newcommand{\bx}{\bm{x}}
\newcommand{\bs}{\bm{s}}
\newcommand{\br}{\bm{r}}
\newcommand{\be}{\bm{e}}
\newcommand{\bk}{\bm{k}}
\newcommand{\bL}{\bm{L}}

\newcommand{\KL}{\mathrm{KL}}
\newcommand{\Matern}{Mat\'ern}
\newcommand{\krel}{\kappa_{\mathrm{rel}}}

\newcommand{\bone}{\mathbf{1}}
\newcommand{\Corr}{\mathrm{Corr}}
\newcommand{\bI}{\mathbf{I}}

\newcommand{\bell}{\boldsymbol{\ell}}

\AtBeginEnvironment{figure}{\singlespacing}
\AtBeginEnvironment{table}{\singlespacing}
\begin{document}

\def\spacingset#1{\renewcommand{\baselinestretch}{#1}\small\normalsize}
\spacingset{1}

\title{\bf Conformal Prediction for Spatially Dependent Data via Sequential Whitening}
\author{
  Ayush Baran Sen\\
  Indian Statistical Institute, Kolkata, India\\
  and\\
  Arkajyoti Saha\thanks{Corresponding author: \texttt{arkajyos@uci.edu}}\\
  Department of Statistics, University of California, Irvine
}
\date{}
\maketitle

\bigskip
\begin{abstract}

Split conformal prediction uses prediction errors on held-out (calibration) data to determine how wide the prediction intervals should be. It guarantees distribution-free finite-sample coverage when these errors and the error at the target site are exchangeable. This assumption may fail under spatial dependence and nonrandom sampling geometry. Existing spatial methods use fitting residuals to remove the predictable part of spatial variation from calibration and target errors. However, the spatial variation that only the calibration residuals can predict remains in both the target and calibration errors, reducing the efficiency and stability of the interval. We address this by additionally conditioning on the calibration residuals sequentially, which scales to large networks through nearest-neighbour approximations. Under a correct working covariance and an elliptical residual law, the resulting interval has exact finite-sample coverage under any spatial design, and under further conditions it is asymptotically oracle efficient. We also bound coverage loss under covariance misspecification and develop a diagnostic that identifies regions at risk of undercoverage. In simulated data, our method produces narrower and more stable intervals than global and localized state-of-the-art alternatives. In a national PM$_{2.5}$ application, it produces narrower intervals within the network and identifies regions at risk of coverage failure.
\end{abstract}


\section{Introduction}
\label{sec:intro}

Traditional uncertainty quantification approaches often require the assumed parametric model to be correctly specified to guarantee validity of the inference. Conformal prediction uses held-out observations (calibration data) to calibrate the fitted model's errors and uses that to build the prediction intervals. This guarantees finite-sample marginal coverage as long as the errors on these held-out observations and the target location are exchangeable, and specifically does not require the mean model or the error distribution to be correctly specified \citep{angelopoulos2023}.

In this article, we focus on inference for point-referenced spatial data. Such data are common across scientific fields, including environmental science, agriculture, ecology, and epidemiology. Prediction at unobserved locations is crucial here for policy and resource decisions. Some examples are assessing pollution exceedances at unmonitored sites \citep{berrocal2010}, allocating inputs across agricultural fields \citep{oliver2010}, and deciding where to drill for mineral deposits \citep{matheron1963}. The data-generating process in these applications is complex and typically unknown, and hence requires intervals that remain reliable even when the spatial model is imperfect.

Conformal inference may appear to offer a solution, but it is natural to wonder what happens to its exchangeability requirement under spatial dependence. The answer depends on how the locations are sampled. Under random location sampling, the scores can be dependent yet exchangeable. However, ordinary split conformal is inefficient here. It does not use the residuals in the observed data to predict the spatial variation in the target error, resulting in a wider interval that has the same width at every target location. The errors in the held-out data are spatially dependent, which makes the calibration quantile, and the resulting interval, unstable. 

\citet{mao2024,moon2026conformal} offer principled ways to address these problems. For split conformal, they propose to predict the spatial variation in the calibration and target residuals from the fitting residuals. Their interval widths also adapt to the prediction uncertainty since they standardize the errors using their respective predictive standard deviations. However, conditioning only on the fitting set still leaves the spatial variation captured solely by the calibration residuals in both the calibration and target errors, offering only a partial solution to the problems with naive split conformal.

When locations are not sampled at random, the spatial variation left in the calibration residuals can make the target and calibration scores incomparable. This violates the exchangeability assumption required for validity of conformal inference. \citet{mao2024,guan2023,jiang2026} propose to bypass this problem by giving greater weight to calibration points near the target. This makes the threshold more representative of the target neighbourhood, even in the presence of correlation in calibration residuals. Localization relies on fewer, often strongly correlated, observations, which can make the resulting quantile unstable. Moreover, strong localization can also produce an unbounded interval if there is not enough data near the target. Additionally, this mechanism does not account for the correlation captured by the calibration residuals and faces the same issues as the global method.

In this article, we develop a spatial conformal procedure that uses the predictive information in both the fitting and calibration residuals. Our main contributions are:
\begin{enumerate}
\item \emph{Methodology.} We introduce a way to use the predictive power of both the fitting and calibration residuals via sequential conditioning. The resulting interval adapts to prediction uncertainty at the target and is narrower and more stable than intervals from existing spatial conformal methods. When the working covariance adequately represents the joint residual structure, the transformation also restores comparability between the calibration and target scores.
\item \emph{Validity and optimality.} Under a correct working covariance matrix and an elliptical residual distribution, we establish exact finite-sample marginal coverage under any spatial design. Under additional assumptions, the interval is asymptotically as short as the shortest conditionally valid interval.

\item \emph{Robustness and diagnostics.} We bound coverage loss under covariance misspecification and characterize target-specific coverage by comparing the variation remaining in the target error with that in the calibration errors. This characterization yields a diagnostic that can identify regions at risk of coverage failure without knowing the true covariance.
\end{enumerate}

The rest of the paper is organized as follows. Section~\ref{sec:prob_def} reviews conformal prediction for independent data and formulates the spatial prediction problem. Section~\ref{sec:method} introduces our method, and Section~\ref{sec:theory} develops its theoretical guarantees. Sections~\ref{sec:sim} and~\ref{sec:data} present the simulation study and an application to fine-particulate-matter monitoring, respectively. Section~\ref{sec:disc} concludes with a discussion and directions for future research. Proofs and additional results appear in the Supplementary Material.

\section{Problem definition}
\label{sec:prob_def}
\subsection{Conformal prediction for independent data}
\label{sec:prob_def_iid}
Let the data be denoted as $(\bx_i,y_i)\in \R^p\times\R$, $i=1,\ldots,n$, which are independent and identically distributed observations from some distribution $P$. Assume that $f(\bx)=\E(y\mid\bx)$ exists. For any given $\alpha\in(0,1)$ and a new data point $(\bx_0,y_0)\sim P$, we want to compute a prediction interval $\widehat C_\alpha(\bx_0)$ which satisfies $\Prob\{y_0\in\widehat C_\alpha(\bx_0)\}\geq1-\alpha$. The probability is over all the data and the new pair. Ideally, the interval should be valid in finite samples for every $P$ and every fitted $\widehat f$, and be as narrow as possible. Split conformal prediction \citep{papadopoulos2002,vovk2005,lei2018} meets the first validity requirement. The procedure is as follows:

\begin{enumerate}[label=(\alph*)]

\item Partition $\{1,\ldots,n\}$ independently of the data into a fitting set $\mathcal F$ and a calibration set $\mathcal C$ of size $m$, and fit any $\widehat f$ using only $\mathcal F$.

\item For each $i\in\mathcal C$, compute the residual $r_i=y_i-\widehat f(\bx_i)$, which measures how poorly $\widehat f$ predicts the $i$th calibration observation.

\item At $\bx_0$, compute the conformal interval as
\begin{equation}
\label{eqn:split_interval}
\widehat C_\alpha(\bx_0)
=
\widehat f(\bx_0)\pm\widehat q_{1-\alpha},
\qquad
\widehat q_{1-\alpha}
=
|r|_{(\lceil(1-\alpha)(m+1)\rceil)},
\end{equation}
where $|r|_{(1)}\leq\cdots\leq|r|_{(m)}$ are the ordered calibration scores and $|r|_{(m+1)}:=+\infty$. The interval is therefore unbounded when $\alpha<1/(m+1)$.
\end{enumerate}

The following lemma \citep{papadopoulos2002,lei2018} ensures finite-sample coverage, as long as the calibration and target scores are \emph{exchangeable}, that is, their joint distribution is invariant to any permutation.

\begin{lemma}
\label{lem:split}
Let $S_i=|r_i|$ for $i\in\mathcal C$ and
$S_0=|y_0-\widehat f(\bx_0)|$. If these $m+1$ scores are exchangeable, then the interval in
\eqref{eqn:split_interval} satisfies
$\Prob\{y_0\in\widehat C_\alpha(\bx_0)\}\geq1-\alpha$.
If the scores are also almost surely distinct, then
\begin{equation}
\label{eqn:split_guarantee}
1-\alpha
\leq
\Prob\{y_0\in\widehat C_\alpha(\bx_0)\}
\leq
1-\alpha+\frac{1}{m+1}.
\end{equation}
\end{lemma}

When scores are distinct, exchangeability implies that the rank of the target score is uniformly distributed among the $m+1$ possible ranks, giving the lower bound in \eqref{eqn:split_guarantee}.  The upper bound comes from rounding the conformal quantile. Exchangeability is trivially satisfied under i.i.d.\ data. We next examine how this argument changes under spatial dependence.

\subsection{What changes for spatial data}
\label{sec:prob_def_spatial}
Suppose we observe $\mathcal D_n=\{(\bx_i,\bs_i,y_i):i=1,\ldots,n\}$, where $\bs_i\in\R^2$ is the spatial location, typically geographical coordinates. We want to predict at a new location $(\bx_0,\bs_0)$. We assume the model
\begin{equation}
\label{eqn:model}
y_i=f(\bx_i)+\eta(\bs_i)+\varepsilon_i,
\qquad
i=0,1,\ldots,n,
\end{equation}
where $f(\bx)$ is the covariate effect and $\eta$ is the spatial structure the covariates do not capture. The process $\eta$ has mean zero and covariance $C_\theta$, with $C_\theta(\bs,\bs)=\sigma^2$. The errors $\varepsilon_i$ are i.i.d.\ with mean zero and variance $\tau^2>0$, and independent of $\eta$. The covariates may vary with location, but where the data are sampled is unrelated to $\eta$ and the errors. We want a prediction interval $\widehat C_\alpha(\bx_0,\bs_0)$ which satisfies $\Prob\{y_0\in\widehat C_\alpha(\bx_0,\bs_0)\}\geq1-\alpha$. Throughout, we treat the locations, the covariates, and the fitted mean and covariance as fixed, and take probabilities over the responses alone. Under random sampling, they are taken over the sampled locations as well.

The residuals are spatially dependent. However, naive conformal inference may still be valid here, since Lemma~\ref{lem:split} only requires exchangeability. Whether exchangeability holds depends on the sampling design. We consider the two cases in turn, starting with the one where it holds.

\subsubsection{When random sampling preserves exchangeability}

\label{sec:prob_def_exch}

Suppose the labelled and target data are sampled independently from the same design, and the fitting--calibration split is independent of the data. Then calibration and target scores are exchangeable \citep[Lemma~1]{mao2024}, so Lemma~\ref{lem:split} applies and naive split conformal is valid. However, it ignores the spatial correlation, and so is inefficient in three ways:
\begin{enumerate}[label=(\alph*)] \item \emph{The width is unstable.} In spatially correlated data, residuals from nearby points are positively correlated. This makes the effective sample size of the calibration set smaller than $m$. As a result, the interval width varies more than it would with $m$ independent scores.

\item \emph{Predictable target variation remains.} Observed residuals contain information about the spatial component of the target error. Naive split conformal prediction does not make use of the predictive power of these residuals. This leaves more unexplained variation for the interval to cover, resulting in a wider interval.

\item \emph{Width does not adapt to location.} Naive conformal intervals have the same width at every target. However, in spatial statistics, isolated locations are harder to predict than those surrounded by observations. Ideally, prediction uncertainty should be reflected in the interval width, which is not the case here.

\end{enumerate}

The concerns are partially addressed by the global spatial conformal approach of \citet{mao2024} (GSCP). Let $\br_{\mathcal F}$ be the vector of fitting residuals and $\bSigma_{\mathcal F}^{\mathrm w}$ its working covariance. For a location $\bs$, write $\bk_{\mathcal F}(\bs)$ for its working covariance with the fitting residuals and $\bSigma^{\mathrm w}(\bs,\bs)$ for its working variance. The working kriging predictor and prediction variance are 
\begin{equation}
\label{eqn:kriging}
\widehat\mu_{\mathcal F}(\bs)
=
\bk_{\mathcal F}(\bs)^\top
(\bSigma_{\mathcal F}^{\mathrm w})^{-1}\br_{\mathcal F},
\qquad
\widehat\nu_{\mathcal F}(\bs)
=
\bSigma^{\mathrm w}(\bs,\bs)
-
\bk_{\mathcal F}(\bs)^\top
(\bSigma_{\mathcal F}^{\mathrm w})^{-1}\bk_{\mathcal F}(\bs).
\end{equation}

GSCP uses the calibration scores $|r_i-\widehat\mu_{\mathcal F}(\bs_i)|/\widehat\nu_{\mathcal F}(\bs_i)^{1/2}$. Let $\widehat q^{\,\mathrm G}_{1-\alpha}$ denote their $\lceil(1-\alpha)(m+1)\rceil$-th order statistic, using the same $+\infty$ convention as in \eqref{eqn:split_interval}. The resulting interval is
\begin{equation}
\label{eqn:gscp}
\widehat C_\alpha^{\mathrm G}(\bx_0,\bs_0)
=
\widehat f(\bx_0)
+
\widehat\mu_{\mathcal F}(\bs_0)
\pm
\widehat q^{\,\mathrm G}_{1-\alpha}
\widehat\nu_{\mathcal F}(\bs_0)^{1/2}.
\end{equation}
GSCP only accounts for the predictive information in the fitting residuals. It does not utilize the spatial information the calibration residuals carry on their own. Hence, the resulting intervals are wider and less stable than they need to be. Though the width of the GSCP intervals adapts to the prediction uncertainty, it applies the same rule to every calibration point and to the target. Therefore, it has the same coverage guarantee as naive split conformal, i.e., it is valid when the sampling preserves exchangeability.

\subsubsection{When geometry breaks exchangeability}
\label{sec:prob_def_nonexch}
There are two separate ways in which exchangeability can fail. First, the target score may have a different marginal distribution from the calibration scores. This might happen under spatial extrapolation or if the locations of the target and calibration are in different regions with distinct residual behaviour \citep{sampson1992}. Second, even with identical marginal distributions, the calibration scores may be more strongly or weakly related to one another than to the target score.

\citet{mao2024,guan2023,jiang2026} propose to address these by giving greater weight to calibration observations near the target. This makes the calibration scores more representative of the target score. In this article, we use a split form of the localized GSCP procedure proposed by \citet{mao2024}, which we call LSCP. A bandwidth parameter $h$ controls the degree of localization. Under strong localization there may be too few calibration observations with a relevant weight, which leads to an unbounded interval. For larger $h$, the weights become more uniform and LSCP approaches GSCP.

GSCP and LSCP both ignore the predictive information in the calibration scores, hence their performance depends on the fitting--calibration split. We see this in the following example. We consider $100$ labelled observations in two well-separated clusters and place the target in one of them. The residuals follow a stationary Gaussian process with strong dependence within clusters and negligible dependence between them. We compare three splits: the calibration observations lie entirely in the target cluster, are divided equally between the clusters, or lie entirely in the distant cluster. We call these the nearby, mixed, and distant calibration designs. We compute the coverage, mean width, and width stability for naive split conformal, GSCP and LSCP under known mean and covariance in Figure~\ref{fig:twoball}. We measure stability using the coefficient-of-variation ratio (CVR). This measures how much larger the coefficient of variation of a method's interval width is than that of naive conformal on an i.i.d.\ version of the data. If the CVR is one, it means that the dependence does not add any extra variability. If it is higher, it means that it does.  The figure also shows the performance of the proposed whitened method introduced in Section~\ref{sec:method}. 

\begin{figure}[t]
\begin{singlespace}
\centering
\includegraphics[width=0.98\textwidth]{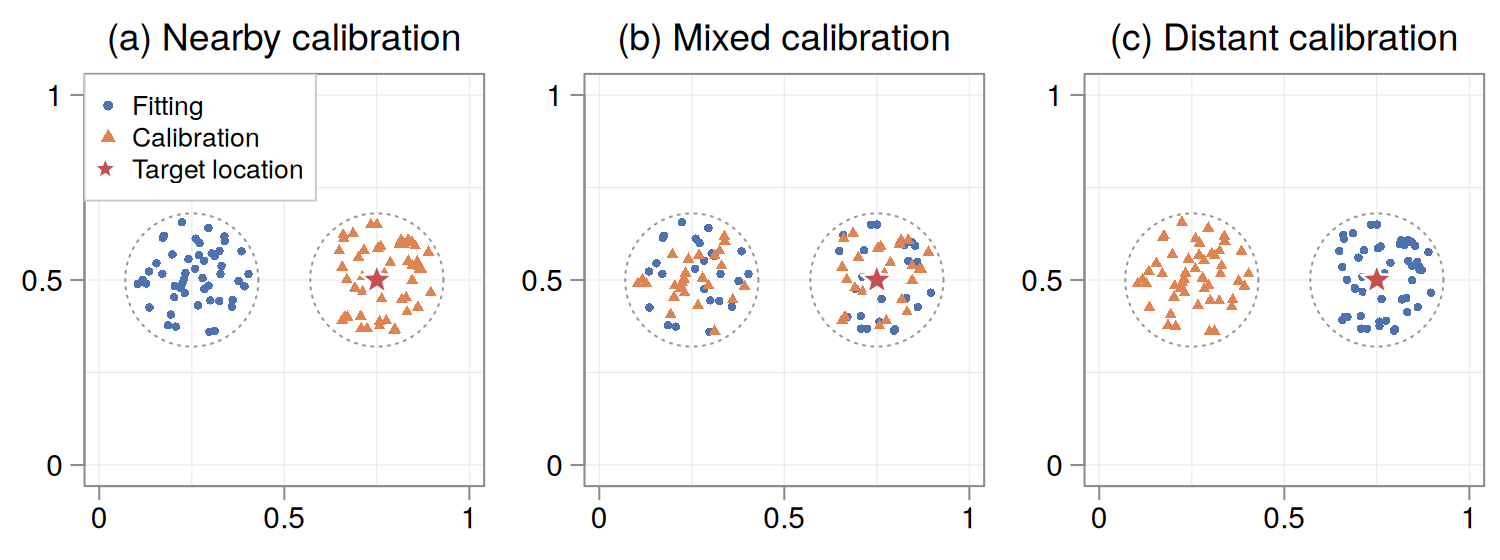}
\small
\setlength{\tabcolsep}{4pt}
\begin{tabular}{lccc}
\toprule
Method & Nearby & Mixed & Distant\\
\midrule
Naive
& $0.911$; $3.32$ $(2.17)$
& $0.908$; $3.41$ $(1.89)$
& $0.858$; $3.30$ $(2.16)$\\
GSCP
& $0.910$; $3.30$ $(2.16)$
& $0.904$; $2.10$ $(1.10)$
& $0.855$; $1.74$ $(2.14)$\\
LSCP, $h=0.05$
& $1.000$; $\infty$
& $1.000$; $\infty$
& $1.000$; $\infty$\\
LSCP, $h=0.20$
& $0.926$; $3.44$ $(2.13)$
& $0.926$; $2.32$ $(1.50)$
& $1.000$; $\infty$\\
LSCP, $h=0.40$
& $0.926$; $3.43$ $(2.13)$
& $0.921$; $2.22$ $(1.17)$
& $0.886$; $1.90$ $(2.07)$\\
LSCP, $h=1$
& $0.926$; $3.43$ $(2.13)$
& $0.919$; $2.20$ $(1.12)$
& $0.873$; $1.83$ $(2.09)$\\
LSCP, $h=2$
& $0.926$; $3.43$ $(2.13)$
& $0.919$; $2.20$ $(1.12)$
& $0.872$; $1.82$ $(2.09)$\\
\midrule
Proposed (Whitened)
& $0.904$; $1.87$ $(1.00)$
& $0.902$; $1.87$ $(1.01)$
& $0.900$; $1.87$ $(1.00)$\\
\bottomrule
\end{tabular}
\normalsize
\caption{The top panel shows the three fitting--calibration splits on fixed data in the two-cluster design. The table below reports sensitivity of naive split conformal, GSCP, LSCP (under different bandwidths $h$), and the proposed approach to the fitting--calibration split. Each entry reports coverage; mean width (CVR) over $20{,}000$ datasets for nominal $90\%$ intervals. Under strong localization, LSCP can produce infinite intervals.}
\label{fig:twoball}
\end{singlespace}
\end{figure}
\vspace{-0.4cm}
Naive split conformal does not use fitting or calibration residuals for prediction, hence its width is not affected by the split and is comparable across the three splits. Its CVR remains near two because it does not account for the correlation among the calibration data. Under nearby calibration, fitting and target data lie in different clusters and are nearly independent. Hence conditioning on fitting data does not give GSCP any additional advantage and it performs similarly to naive conformal. For mixed calibration, since fitting data share a cluster with both the calibration and target data, conditioning on them can explain some of the predictable spatial variation in both the target and the calibration data, leading to a much narrower interval with a CVR that is closer to one. In distant calibration the fitting data share a cluster with the target, hence GSCP can explain a lot more of the predictable variation in the target residual, making the interval the narrowest. However, the fitting data are nearly independent of the calibration data, and cannot explain the correlation in the calibration residuals, which reintroduces instability in the width and the CVR jumps to near two.

As far as coverage is concerned, naive conformal and GSCP attain near-nominal coverage under both nearby and mixed calibration, where all or part of the calibration data lie in the same cluster as the target. Under distant calibration, the calibration data are in a different cluster from the target. The calibration scores are then strongly related to one another but nearly independent of the target score. This violates exchangeability, and both naive conformal and GSCP undercover.

To see the problem, Figure~\ref{fig:regimes} plots the rank of the target score among the calibration scores under distant calibration and compares it with nearby calibration. For GSCP the ranks are nearly uniform under nearby calibration but bimodal under distant calibration. The bimodal case places excess probability in the noncoverage region, so GSCP undercovers.

\begin{figure}[t]
\centering
\includegraphics[width=0.94\textwidth]{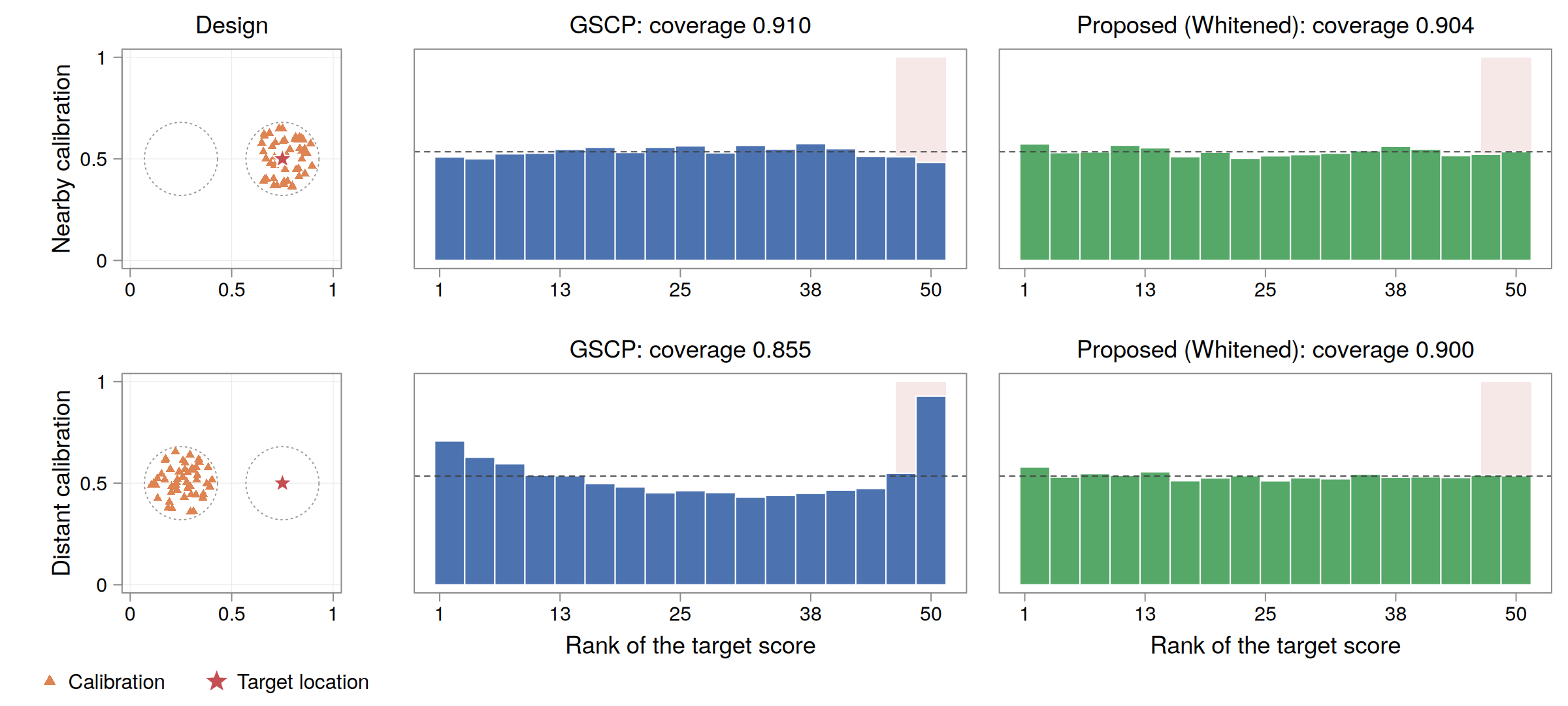}
\caption{Distribution of rank of target score in the two-cluster example. The top row uses nearby calibration, and the bottom row uses distant calibration. The left column shows the calibration and target locations; the middle and right columns show the distribution of the rank of the target score among calibration scores under GSCP and the proposed whitened method. Dashed lines mark the expected bin height under a uniform rank distribution, and shaded regions contain ranks producing noncoverage.}
\label{fig:regimes}
\end{figure}
\vspace{-0.4cm}
\citet{mao2024} propose to bypass this problem via localization in LSCP.  However, the performance of LSCP depends on the bandwidth that controls the degree of localization. In Figure~\ref{fig:twoball}, very strong localization (small $h$) produces an unbounded interval. An intermediate $h$ increases coverage, but only with a wider interval, and overcovers where GSCP was already calibrated. As $h$ increases further, the interval narrows toward GSCP and still undercovers under distant calibration. The $h$ that balances coverage and width is unknown in practice. 
The example shows that conditioning on the fitting residuals alone leads to an interval that is wider than necessary and less stable, and can cause undercoverage or unbounded intervals. In the next section we propose our method, which addresses all of these together.

\section{Method}
\label{sec:method}
The previous section showed that the crux of the issue is that the predictive information in the calibration residuals goes unused. However, predicting a calibration residual from the entire calibration set makes it predict itself. Its predictive variance is then zero and its score degenerate, so it can no longer be compared with the target score.

We propose to bypass this by conditioning on the calibration residuals sequentially. Fixing an order that does not depend on the responses, we predict each calibration residual from the fitting residuals and the calibration residuals before it. We subtract this prediction and divide by its predictive standard deviation. We place the target last and predict it from both the fitting and calibration residuals. We formalize the method below. 

\subsection{Sequential prediction and standardization}
\label{sec:method_sequential}
Let $\mathcal F$ and $\mathcal C$ be the fitting and calibration sets of sizes $n_{\mathcal F}$ and $m$, respectively. We fit the mean model $\widehat f$ on $\mathcal F$, and compute out-of-bag or cross-fitted residuals, $\br_{\mathcal F}$. For $i\in\mathcal C$, let $r_i=y_i-\widehat f(\bx_i)$, and $r_0=y_0-\widehat f(\bx_0)$ be the calibration and target residuals, respectively. Here $i = 1, 2, \ldots, m$ indexes any response-independent order, and we write $\br_{\mathcal C}=(r_1,\ldots,r_m)^\top$. Next, we stack the residuals as follows:
\begin{equation}
\label{eq:stacked_r}
\br_+
=
\bigl(
\br_{\mathcal F}^{\top},
\br_{\mathcal C}^{\top},
r_{m+1}
\bigr)^{\top}\in\R^N,
\qquad r_{m+1}=r_0,\quad N=n_{\mathcal F}+m+1.
\end{equation}
Let $\bSigma_+^{\mathrm w}$ be a positive-definite working covariance matrix for $\br_+$. For $A,B\subseteq\{1,\ldots,N\}$, write $\br_{+,A}$ and $\bSigma_{A,B}^{\mathrm w}=(\bSigma_+^{\mathrm w})_{A,B}$ for the corresponding subvector and submatrix. For each $i \in \{1, 2,\ldots, m+1 \}$, let $k_i=n_{\mathcal F}+i$ and $I_i=\{1,\ldots,k_i-1\}$. The standardized score of $r_i$ is
\begin{equation}
\label{eq:innovation}
e_i=\frac{r_i-\widehat\mu_i}{\widehat\nu_i^{1/2}},
\end{equation}
where, for all $i \in \{1, 2,\ldots, m+1 \}$,
\begin{equation}
\label{eq:seq_musd}
\widehat\mu_i=\bSigma_{k_i,I_i}^{\mathrm w}\bigl(\bSigma_{I_i,I_i}^{\mathrm w}\bigr)^{-1}\br_{+,I_i},
\qquad
\widehat\nu_i=\bSigma_{k_i,k_i}^{\mathrm w}-\bSigma_{k_i,I_i}^{\mathrm w}\bigl(\bSigma_{I_i,I_i}^{\mathrm w}\bigr)^{-1}\bSigma_{I_i,k_i}^{\mathrm w}.
\end{equation}

\subsection{Equivalence to triangular whitening}
\label{sec:method_op} 
Write the Cholesky factorization of the working covariance as $\bSigma_+^{\mathrm w}=\bL\bL^\top$, where $\bL=(\ell_{ij})_{i,j=1}^N$ is lower triangular with positive diagonal entries. Next, we connect the sequential construction to a single matrix transformation.

\begin{lemma}[Sequential form of triangular whitening]
\label{lem:sequential_whitening}
For every $i=1,\ldots,m+1$,
\begin{equation}
\label{eq:sequential_whitening}
\bigl(\bL^{-1}\br_+\bigr)_{n_{\mathcal F}+i}
=
e_i.
\end{equation}
\end{lemma}

The proof is given in Supplementary Material Section~\ref{supp:sec_whitening}. This result shows that the  sequential prediction and standardization are equivalent to the triangular whitening of the ordered residual vector. 

\subsection{The whitened conformal interval}
\label{sec:method_interval}
We now construct the whitened conformal interval.

\begin{theorem}[Finite-sample validity of the whitened interval] \label{thm:whitened_interval} Fix $\alpha\in(0,1)$. Suppose that $\br_+$ has a centred, absolutely continuous elliptical distribution with scatter matrix $\bSigma_+^{\mathrm w}$. Let $j_\alpha=\lceil(1-\alpha)(m+1)\rceil$ and let $\widehat q^{\,\mathrm W}_{1-\alpha}$ be the $j_\alpha$-th order statistic of $\{|e_1|,\ldots,|e_m|,+\infty\}$, and write $\widehat\mu_{\mathcal F\cup\mathcal C}(\bs_0)=\widehat\mu_{m+1}$ and $\widehat\nu_{\mathcal F\cup\mathcal C}(\bs_0)=\widehat\nu_{m+1}$ from \eqref{eq:seq_musd} at $i=m+1$. Then the whitened conformal interval
\begin{equation}
\label{eq:interval}
\widehat C_\alpha^{\mathrm W}(\bx_0,\bs_0)
=
\widehat f(\bx_0)
+
\widehat\mu_{\mathcal F\cup\mathcal C}(\bs_0)
\pm
\widehat q^{\,\mathrm W}_{1-\alpha}
\widehat\nu_{\mathcal F\cup\mathcal C}(\bs_0)^{1/2}
\end{equation}
satisfies
$\Prob\{y_0\in\widehat C_\alpha^{\mathrm W}(\bx_0,\bs_0)\}
=\lceil(1-\alpha)(m+1)\rceil/(m+1)\geq1-\alpha$.
\end{theorem}
Refer to Supplementary Material Section~\ref{supp:sec_exact} for the proof. The result also holds when $\widehat f$ and $\bSigma^{\mathrm w}_+$ are estimated from the fitting data, provided that, given the fitting data, the calibration and target residuals have an absolutely continuous elliptical distribution with location equal to their working kriging predictors from $\br_{\mathcal F}$ and scatter matrix equal to the corresponding working conditional covariance.

\subsection{What whitening changes}
\label{sec:method_gain}
The two-cluster example in Section~\ref{sec:prob_def_nonexch} showed that the nonoptimality of GSCP and LSCP comes from conditioning on the fitting residuals alone. Whitening addresses this by conditioning on both the fitting and calibration residuals.
\paragraph{\emph{Restoring rank symmetry.}}
Under a correct working covariance and an elliptical residual law, accounting for the correlation within calibration residuals makes calibration and target scores exchangeable. Figure~\ref{fig:regimes} shows this. The GSCP ranks are bimodal under distant calibration, whereas they become nearly uniform under our proposed approach.

\paragraph{\emph{Reducing interval width.}}
Since the whitened interval uses more information than GSCP to predict the target errors, intuitively our interval should be narrower. Proposition~\ref{prop:width} makes this precise: the ratio of the whitened and GSCP interval widths is asymptotically equal to the ratio of their predictive standard deviations at the target. Since $\mathcal F\subseteq\mathcal F\cup\mathcal C$, conditioning on the calibration residuals cannot increase the prediction variance, therefore $\widehat\nu_{\mathcal F\cup\mathcal C}(\bs_0)\le\widehat\nu_{\mathcal F}(\bs_0)$. When calibration residuals contain information in addition to fitting data, this inequality is strict. Figure~\ref{fig:twoball} shows this in the two-cluster case.
\begin{proposition}[Asymptotic width ratio]
\label{prop:width}
Let $W^{\mathrm G}(\bs_0)$ and $W^{\mathrm W}(\bs_0)$ be the widths of the GSCP and whitened intervals
under the same working covariance. Suppose that, as $m\to\infty$, the empirical distributions of the absolute calibration scores of the two methods converge in probability to the same limit $G_0$. Suppose also that $G_0$ is continuous and strictly increasing at its $(1-\alpha)$ quantile $\zeta_{1-\alpha}>0$. Then
\begin{equation}
\label{eq:width_ratio}
\frac{W^{\mathrm W}(\bs_0)}{W^{\mathrm G}(\bs_0)}
=
\left\{
\frac{\widehat\nu_{\mathcal F\cup\mathcal C}(\bs_0)}{\widehat\nu_{\mathcal F}(\bs_0)}
\right\}^{1/2}
\{1+o_{\Prob}(1)\}.
\end{equation}
\end{proposition}
The proof is given in Supplementary Material Section~\ref{supp:sec_eff}.

\paragraph{\emph{Stabilizing the calibration threshold.}}
Under a correct Gaussian working covariance, the sequential conditioning of whitening removes correlation within calibration residuals beyond what is removed by GSCP and LSCP. As a result, each score contributes separate information to the threshold. Empirically, this makes our method's stability comparable to the independent-data case. Figure~\ref{fig:twoball} demonstrates this in the two-cluster example.

\subsection{Implementation}
\label{sec:method_impl}
\begin{algorithm}[!t]
\caption{Whitened spatial conformal prediction}
\label{alg}
\begin{algorithmic}[1]
\Require Labelled data $\mathcal D_n=\{(\bx_i,\bs_i,y_i)\}_{i=1}^n$, target $(\bx_0,\bs_0)$, fitting and calibration sets $(\mathcal F,\mathcal C)$, miscoverage level $\alpha$, mean estimator $\mathcal A$, working covariance model $C_\theta^{\mathrm w}$, covariance-parameter estimator $\mathcal E$, calibration-ordering rule $\mathcal O$, and routine $\mathcal B$ that returns the inverse-Cholesky factor of the fitted working covariance at an ordered set of locations

\State Fit $\widehat f\gets\mathcal A(\mathcal F)$

\State Compute fitting residuals $\br_{\mathcal F}$, using out-of-bag or cross-fitted predictions for flexible learners, and compute $r_i=y_i-\widehat f(\bx_i)$ for $i\in\mathcal C$

\State Estimate $\widehat\theta
\gets
\mathcal E\!\left(
\{(\bs_i,(\br_{\mathcal F})_i):i\in\mathcal F\};
C_\theta^{\mathrm w}
\right)$

\State Set $\pi_{\mathcal C}\gets\mathcal O(\{(\bx_i,\bs_i):i\in\mathcal C\})$ and form $\br_{\mathcal C}
=
\bigl(
r_{\pi_{\mathcal C}(1)},\ldots,r_{\pi_{\mathcal C}(m)}
\bigr)^\top$

\State Form the augmented residual vector $\br_+=\bigl(\br_{\mathcal F}^{\top},\br_{\mathcal C}^{\top},r_{m+1}\bigr)^{\top}$, where the final entry $r_{m+1}=y_0-\widehat f(\bx_0)$ is unknown

\State Construct the inverse-Cholesky factor of the fitted working covariance of $\br_+$ as $\bL^{-1}
\gets
\mathcal B\!\left(
C_{\widehat\theta}^{\mathrm w};
\bs_{\mathcal F},
\bs_{\pi_{\mathcal C}},
\bs_0
\right)$

\State Compute $(\bL^{-1})_{1:n,1:n}\br_{+,1:n}$ and retain its final $m$ coordinates as $e_1,\ldots,e_m$

\State Set $\widehat q^{\,\mathrm W}_{1-\alpha}$ to the $\lceil(1-\alpha)(m+1)\rceil$-th order statistic of $\{|e_1|,\ldots,|e_m|,+\infty\}$

\State Let $(\bm a^\top,d)$ be the final row of $\bL^{-1}$ and compute
\[
\widehat\mu_{\mathcal F\cup\mathcal C}(\bs_0)
=
-\frac{\bm a^\top\br_{+,1:n}}{d},
\qquad
\widehat\nu_{\mathcal F\cup\mathcal C}(\bs_0)
=
d^{-2}
\]

\State \Return The whitened conformal interval
\[
\widehat C_\alpha^{\mathrm W}(\bx_0,\bs_0)
=
\widehat f(\bx_0)
+
\widehat\mu_{\mathcal F\cup\mathcal C}(\bs_0)
\pm
\widehat q^{\,\mathrm W}_{1-\alpha}
\widehat\nu_{\mathcal F\cup\mathcal C}(\bs_0)^{1/2}.
\]
\end{algorithmic}
\end{algorithm}
\vspace{-0.2cm}

Algorithm~\ref{alg} summarizes the complete construction. The construction allows flexibility in mean estimation, construction of the inverse-Cholesky factor of the working covariance, and calibration ordering. We describe the choices used in this article below.
\paragraph{\emph{Mean estimation.}}
This approach works with any mean estimator. In-sample fitting residuals from flexible learners are artificially small. This distorts covariance
estimation and whitening. Therefore, we build $\br_{\mathcal F}$ from out-of-bag or cross-fitted predictions. Here we use a random forest with out-of-bag residuals.

\paragraph{\emph{Working covariance and factorization.}}
This approach works with any positive-definite covariance model and parameter estimator. We use an exponential covariance with a nugget throughout and
estimate its parameters with the Nearest Neighbor Gaussian Process (NNGP) approximation \citep{datta2016}, implemented in BRISC \citep{saha2018brisc}. We then use these to compute the dense working covariance matrix and its Cholesky factor. This costs cubic time and quadratic storage. For moderate to large samples, this can become computationally prohibitive. Hence we recommend using NNGP covariance there. In NNGP, each residual is conditioned on at most $K$ preceding neighbours, hence the inverse-Cholesky factor has at most $K+1$ nonzero entries per row and costs linear time and storage. The Supplementary Material Section~\ref{supp:impl_ours} shows that the two implementations give practically indistinguishable results.

\paragraph{\emph{Ordering.}}
The method works with any order that does not depend on the responses. We use the order supplied by the split, and the Supplementary Material Section~\ref{supp:impl_ours} shows that the results are insensitive to this choice.

\section{Theoretical guarantees}\label{sec:theory}
In this section, we demonstrate the asymptotic equality of the widths of the whitened interval and the shortest conditionally valid interval under oracle conditions. Next, we consider what happens under misspecification. First, we derive a global bound on the resulting coverage error and then develop a local diagnostic that detects regions at risk of undercoverage without knowing the true data-generating process. Throughout this section, we treat the locations, the fitted mean, and the working covariance as fixed. 

\subsection{Efficiency}
\label{sec:opt}

\begin{theorem}[Asymptotic oracle efficiency]
\label{thm:eff}
Let $\mathcal D=\mathcal F\cup\mathcal C$. Suppose the working covariance is correct and the coordinates of $\be_+=\bL^{-1}\br_+$ are independent with a common distribution having mean zero, unit variance, and a continuous density that is symmetric about zero and strictly decreasing on $(0,\infty)$. Suppose also that the density of $|e|$ is positive at its $(1-\alpha)$ quantile. Let $W^{\mathrm o}(\bs_0)$ denote the width of the shortest interval, measurable with respect to $\br_{\mathcal D}$, having conditional coverage at least $1-\alpha$ given $\br_{\mathcal D}$ when the true joint law is known. If $W^{\mathrm W}(\bs_0)$ denotes the width of the whitened interval in \eqref{eq:interval}, then, as $m\to\infty$,
\begin{equation}
\label{eq:eff}
\frac{W^{\mathrm W}(\bs_0)} {W^{\mathrm o}(\bs_0)} = 1+O_{\Prob}(m^{-1/2}).
\end{equation}
\end{theorem}

The oracle and whitened intervals share the same predictive standard deviation but have different multipliers. The whitened interval uses the empirical quantile of the whitened calibration scores, whereas the oracle uses the population quantile of $|e|$. This empirical calibration quantile approaches the oracle quantile at the rate in \eqref{eq:eff} under the assumptions of the theorem. Supplementary Material Section~\ref{supp:sec_eff} contains the proof.

\subsection{Coverage under covariance misspecification}
\label{sec:theory_gap}

In reality, the true covariance and the covariance utilised for whitening are typically different. First, we use the difference between an exchangeable reference law and the actual law of the transformed calibration and target errors to bound the resulting coverage error.

Let $\be=(e_1,\ldots,e_m,e_{m+1})^\top$ be the vector of transformed calibration and target errors. Let $Q_{\be}$ represent a reference law under which the absolute coordinates of $\be$ are exchangeable and almost surely distinct, and let $P_{\be}$ denote the actual law of $\be$.

\begin{theorem}[Coverage gap]
\label{thm:gap}
If $P_{\be}\ll Q_{\be}$, then
\begin{equation}
\label{eq:coverage_gap_kl}
\left| \Prob\{y_0\in\widehat C^{\mathrm W}_\alpha(\bx_0,\bs_0)\} -(1-\alpha) \right| \leq \frac{1}{m+1} + \sqrt{\frac{\KL(P_{\be}\,\|\,Q_{\be})}{2}}.
\end{equation}
\end{theorem}

Exchangeability places the coverage probability within $1/(m+1)$ of its nominal level under $Q_{\be}$. When we replace $Q_{\be}$ with the real law $P_{\be}$, the probability changes by at most the total variation distance between the two laws. Applying Pinsker's inequality then gives us \eqref{eq:coverage_gap_kl}. Supplementary Material Section~\ref{supp:sec_gap} contains the proof.

Now assume that the true and working residual laws have positive-definite covariance matrices $\Sig_+$ and $\bSigma^{\mathrm w}_+=\bL\bL^\top$, respectively, and are centred Gaussian. Define $\Omega=\bigl(\bL^{-1}\Sig_+\bL^{-\top}\bigr)_{J,J}$, the actual covariance matrix of the transformed calibration and target errors, where $J=\{n_{\mathcal F}+1,\ldots,N\}$. Their covariance matrix under the working Gaussian reference is $I_{m+1}$. If $\rho_1,\ldots,\rho_{m+1}$ are $\Omega$'s eigenvalues, then $\KL\{N(\bm 0,\Omega)\,\|\,N(\bm 0,I_{m+1})\}=\frac12\sum_{j}(\rho_j-1-\log\rho_j)$, and substituting into \eqref{eq:coverage_gap_kl} yields
\begin{equation}
\label{eq:coverage_gap_gaussian}
\left| \Prob\{y_0\in\widehat C^{\mathrm W}_\alpha(\bx_0,\bs_0)\} -(1-\alpha) \right| \leq \frac{1}{m+1} + \frac12 \left\{ \sum_{j=1}^{m+1} \bigl(\rho_j-1-\log\rho_j\bigr) \right\}^{1/2}.
\end{equation}

\subsection{A local coverage diagnostic}
\label{sec:theory_loc}
At a given target, \eqref{eq:coverage_gap_gaussian} does not give us information about whether the coverage is higher or lower than nominal. This is determined by the ratio of the scale of the transformed target error to that of the calibration scores.

\begin{theorem}[Relative-scale characterization]
\label{thm:loc}
Suppose the residual process is centred and Gaussian. Write $z_0=e_{m+1}$, let $\kappa_i^2=\E(e_i^2)$
for $i=1,\ldots,m$ and $\kappa(\bs_0)^2=\E(z_0^2)$, and let $G$ be the standard half-normal
distribution function with $\xi_{1-\alpha}=G^{-1}(1-\alpha)$. Define
\begin{equation}
\label{eq:relative_scale}
\overline F_m(t)=\frac1m\sum_{i=1}^m G\!\left(\frac{t}{\kappa_i}\right),\qquad q_m^\star=\overline F_m^{-1}(1-\alpha),\qquad \krel(\bs_0)=\frac{\xi_{1-\alpha}\,\kappa(\bs_0)}{q_m^\star},
\end{equation}
and let $v_m=m^{-2}\sum_{i,j=1}^m\Corr(e_i,e_j)^2$. Suppose the calibration scales $\kappa_i$ are
bounded above and away from zero, $\kappa(\bs_0)$ is bounded away from zero, and $v_m\to0$. Then, as
$m\to\infty$,
\begin{equation}
\label{eq:loc}
\Prob\{y_0\in\widehat C^{\mathrm W}_\alpha(\bx_0,\bs_0)\} = G\!\left\{\frac{\xi_{1-\alpha}}{\krel(\bs_0)}\right\} +O\bigl(v_m^{1/3}\bigr).
\end{equation}
\end{theorem}
The result follows by replacing the conformal threshold with its large-sample value $q_m^\star$, which is the $(1-\alpha)$ quantile of the average distribution of the absolute calibration errors. The condition $v_m\to0$ requires enough calibration information to remain after whitening for the threshold to concentrate. It permits the target and calibration errors to be dependent, permits different calibration scales, and does not impose stationarity or ordering. Section~\ref{supp:sec_loc} of the Supplementary Material contains the proof. 

We call $\krel(\bs_0)$ in \eqref{eq:relative_scale} the relative predictive scale. Equation~\eqref{eq:loc} implies approximately nominal coverage when $\krel(\bs_0)=1$, conservative coverage when $\krel(\bs_0)<1$, and undercoverage when $\krel(\bs_0)>1$. The pointwise relative scale is generally unknown. But held-out responses from the same region offer a regional approximation. We use the root-mean-square transformed error among the regional held-out data to estimate $\kappa(\bs_0)$. We then compare it to the conformal threshold that the interval actually employs. This diagnostic identifies areas that may be undercovered or overcovered without requiring the true covariance matrix. For technical details of implementation, please refer to Supplementary Material Section~\ref{supp:impl_prosp}. 

\section{Simulation study}
\label{sec:sim}

The simulations address three questions. \textbf{(Q1)} Does whitening result in a narrower and more stable interval in an exchangeable regime when all methods retain nominal coverage? \textbf{(Q2)} Does the benefit of whitening increase when the calibration residuals provide more information about the target than the fitting residuals? \textbf{(Q3)} When the calibration and the target data come from different environments, does that lead to a coverage failure? If so, can the proposed prospective diagnostic identify this failure in advance?

\subsection{Methods and designs}

\label{sec:sim_setup}

The three questions use different sampling designs and allocations of labelled observations between the fitting and calibration sets. \emph{Regime A} addresses \textbf{Q1} by allocating them at random on a uniform sampling design, which guarantees that the design maintains exchangeability. 

One can use any response-independent allocation of the fitting and calibration sets, but random splitting does not cover all practically relevant instances. In many applications, the split is predetermined because a model has already been trained on past data and later observations are only available for calibration. Even when we can choose the split, a single fitting--calibration split may be used for multiple targets, and for some of them the points in the calibration set may be the most informative. As shown in the two-cluster example (Figure~\ref{fig:twoball}), the performance of GSCP and LSCP can depend strongly on this allocation, since they do not make use of the predictive information in the calibration data. Since whitening makes use of this information, under the oracle setup of the two-cluster example, it was shown to be largely unaffected by the choice of the split, while also producing narrower and more stable intervals. 

\emph{Regimes B and C} investigate the effect on performance when the sampling and splitting are nonrandom. \emph{Regime B} addresses \textbf{Q2} by allocating the points most informative about the target to the calibration set to varying degrees, which isolates the gain of utilizing the predictive information in the calibration residuals in whitening. \emph{Regime C} addresses \textbf{Q3} by creating dissimilarity between the calibration and target prediction environments through nonrandom spatial geometry, extrapolation, and nonstationarity. It evaluates both the resulting coverage and the ability of the proposed diagnostic to anticipate target-specific failures. 

We compare four methods: (1) naive split conformal prediction from \eqref{eqn:split_interval}; (2) the split version of GSCP in \eqref{eqn:gscp}, with the covariance estimator of the \texttt{scp} package accompanying \citet{mao2024}; (3) its localized version, LSCP, with bandwidth $h=0.40$; and (4) the proposed whitened interval from \eqref{eq:interval}, using Algorithm~\ref{alg}. 

For each simulation setup, we use $500$ independently generated data sets of $501$ spatial sites. For each site, we draw five i.i.d.\ $\operatorname{Unif}(0,1)$ covariates independently of location, and compute its response as $y(\bs)=\{10\sin(\pi x_1x_2)+20(x_3-0.5)^2+10x_4+5x_5\}/6+r(\bs)$,
where $r(\bs)$ is a mean-zero spatial field. Five sites are held out as prediction targets, leaving $496$ labelled observations, of which, unless stated otherwise, $330$ form the fitting set and $m=166$ the calibration set.

Unless stated otherwise, $r(\bs)$ is a Gaussian process with covariance
\begin{equation}
\label{eq:matern}
\operatorname{Cov}\{r(\bs),r(\bs')\}=\sigma^2\rho(\|\bs-\bs'\|)+\tau^2\mathbf 1\{\bs=\bs'\},
\qquad
\rho(d)=\frac{2^{1-\nu}}{\Gamma(\nu)}\bigl(\sqrt{2\nu}\,\phi d\bigr)^{\nu}K_\nu\bigl(\sqrt{2\nu}\,\phi d\bigr),
\end{equation}
with $\sigma^2=1$ and $\tau^2=0.1$, where $K_\nu$ is the modified Bessel function of the second kind. Here $\nu$ controls the smoothness of the process, with larger values giving smoother sample paths, and $\phi$ is an inverse-range parameter, with larger values giving faster decay. The case $\nu=1/2$ is the exponential correlation. We also use the Gaussian correlation $\rho(d)=e^{-(\phi d)^2}$.

Please see Supplementary Material Section~\ref{supp:sec_impl} for more details, including the regime-specific data-generating processes, the detailed implementation of every method, and the sensitivity of the comparison to the implementation choices.

\subsection{Results}
\paragraph{\emph{Regime A: efficiency and stability under random sampling.}}
Regime A considers nine different settings with stationary residual processes under random locations and random splits. The first six settings use an exponential covariance and vary the decay parameter $\phi$. Smaller values of $\phi$ produce stronger and longer-range dependence. The final three settings hold $\phi = 6$ fixed and increase the smoothness, which makes the field smoother at short distances. We compute coverage and mean width by averaging over random designs (i.e., the locations are redrawn in each replicate). We measure stability using the CVR defined in Section~\ref{sec:prob_def_nonexch}. To isolate the natural width variation of the methods from design effects, we compute the CVR on a fixed-design counterpart of each setting, in which one configuration of locations and split is held fixed across replicates.

\begin{table}[t]
\centering
\small
\setlength{\tabcolsep}{3pt}
\caption{Efficiency and stability under random sampling. Each entry reports coverage/mean width, with the CVR from fixed-design counterparts in parentheses. Nominal coverage is $0.90$.}
\label{tab:exch}
\begin{tabular}{lcccc}
\toprule
Design & Naive & GSCP & LSCP & Whitened\\
\midrule
Exponential, $\phi=0.5$
& $.899/2.19\ (2.33)$
& $.902/1.75\ (1.91)$
& $.909/1.80\ (1.89)$
& $.900/\mathbf{1.66}\ (\mathbf{1.12})$\\
Exponential, $\phi=1$
& $.907/2.59\ (2.54)$
& $.910/1.84\ (2.00)$
& $.918/1.90\ (1.79)$
& $.916/\mathbf{1.76}\ (\mathbf{1.10})$\\
Exponential, $\phi=2$
& $.908/3.02\ (2.25)$
& $.918/2.01\ (1.36)$
& $.922/2.07\ (1.47)$
& $.907/\mathbf{1.93}\ (\mathbf{1.11})$\\
Exponential, $\phi=6$
& $.910/3.59\ (1.87)$
& $.902/2.42\ (1.21)$
& $.911/2.50\ (1.26)$
& $.903/\mathbf{2.32}\ (\mathbf{1.16})$\\
Exponential, $\phi=8$
& $.907/3.64\ (1.54)$
& $.906/2.55\ (1.17)$
& $.910/2.62\ (1.12)$
& $.903/\mathbf{2.43}\ (\mathbf{1.06})$\\
Exponential, $\phi=30$
& $.906/3.74\ (1.07)$
& $.900/3.37\ (\mathbf{1.05})$
& $.911/3.47\ (1.08)$
& $.902/\mathbf{3.24}\ (\mathbf{1.05})$\\
\Matern{}, $\nu=3/2$
& $.909/3.56\ (1.96)$
& $.913/2.02\ (1.61)$
& $.919/2.08\ (1.55)$
& $.906/\mathbf{1.94}\ (\mathbf{1.31})$\\
\Matern{}, $\nu=5/2$
& $.913/3.55\ (2.23)$
& $.904/1.93\ (1.86)$
& $.909/1.99\ (1.68)$
& $.895/\mathbf{1.85}\ (\mathbf{1.30})$\\
Gaussian
& $.906/3.63\ (1.93)$
& $.906/1.97\ (1.64)$
& $.920/2.03\ (1.82)$
& $.908/\mathbf{1.89}\ (\mathbf{1.24})$\\
\bottomrule
\end{tabular}
\end{table}
\vspace{-0.4cm}
In all settings (Table~\ref{tab:exch}), whitening produces the narrowest interval. Its mean width is approximately $4\%$ and $7\%$ smaller than that of GSCP and LSCP, respectively. In comparison to naive split conformal, whitened intervals are $14\%$ to $48\%$ narrower across designs, with an average of $35\%$. The widths of the whitened intervals are also much more stable than those of the competing methods. The whitened interval has the lowest CVR in every setting (tied with GSCP at $\phi=30$). Its advantage over GSCP and LSCP is largest under strong dependence. Naive conformal has the highest CVR in the presence of moderate to strong spatial dependence. This is expected, since it does not account for any spatial information. The methods perform comparably in terms of stability under weak dependence ($\phi=30$), where all of them approach the independent-data benchmark of one. Therefore, Regime A shows that even under random sampling and splitting, where whitening has the least to gain, it produces the narrowest and most stable intervals. 

We only use Gaussian processes as residuals here. However, Theorem~\ref{thm:whitened_interval} covers any elliptical residual law, and Supplementary Material Section~\ref{supp:impl_ngauss} reports results under a spatial $t$ process.

\paragraph{\emph{Regime B: the value of predictive calibration residuals.}}
Regime B considers six settings in which the calibration set is, to different degrees, more informative about the target set than the fitting set. In the three halo designs shown in the upper row of Figure~\ref{fig:designB}, the $k$ nearest points to each target are assigned to calibration data. Increasing $k$ therefore shifts more local information about the target from the fitting to the calibration set. The lower row considers three more structured sampling geometries: transect sampling, boundary prediction, and a cluster monitoring network. For each setting, we measure how much smaller the predictive standard deviation at the target is, under the true covariance, when conditioning on both the fitting and calibration sets rather than on the fitting set alone. This is given in Table~\ref{tab:calinfo} under \emph{Available gain}, as it measures how much the additional conditioning on the calibration set helps whitening over GSCP.

\begin{figure}[t]
\centering
\includegraphics[width=\textwidth]{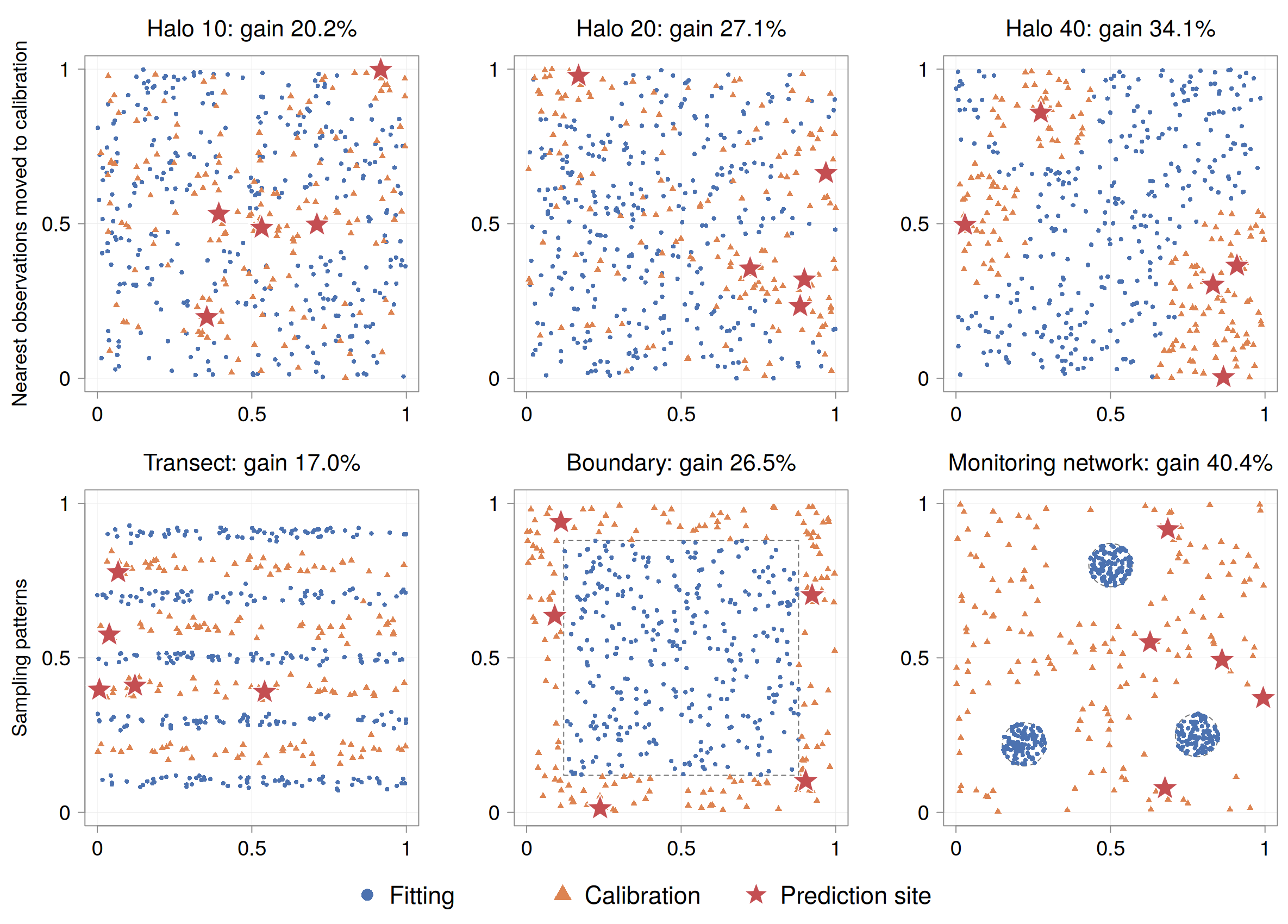}
\caption{Spatial designs used to vary the predictive information contained in the calibration residuals. Blue circles denote fitting observations, orange triangles denote calibration observations, and red stars denote prediction sites. The halo designs assign progressively more observations near each target to calibration, while the lower panels represent transect sampling, boundary prediction, and a cluster monitoring network. Each title reports the available gain.}
\label{fig:designB}
\end{figure}
\begin{table}[t]
\centering
\small
\setlength{\tabcolsep}{2.5pt}
\caption{Value of informative calibration sets. Method entries report coverage/mean width. The last two columns give the available gain and the achieved width reduction of whitening relative to GSCP, both in percent. Nominal coverage is $0.90$.}
\label{tab:calinfo}
\begin{tabular}{lcccccc}
\toprule
Design & Naive & GSCP & LSCP & Whitened & Available gain & Achieved gain\\
\midrule
Halo $10$
& $.900/3.01$ & $.894/2.23$ & $.907/2.30$ & $.898/\mathbf{1.92}$ & $20.2$ & $14.0$\\
Halo $20$
& $.900/3.03$ & $.899/2.35$ & $.910/2.41$ & $.909/\mathbf{1.91}$ & $27.1$ & $18.7$\\
Halo $40$
& $.904/3.12$ & $.908/2.57$ & $.916/2.62$ & $.913/\mathbf{1.93}$ & $34.1$ & $24.8$\\
Transect
& $.904/2.95$ & $.902/2.25$ & $.912/2.30$ & $.903/\mathbf{1.92}$ & $17.0$ & $14.6$\\
Boundary
& $.890/3.27$ & $.894/2.48$ & $.910/2.57$ & $.899/\mathbf{1.97}$ & $26.5$ & $20.3$\\
Monitoring network
& $.905/3.23$ & $.899/3.06$ & $.917/3.14$ & $.918/\mathbf{2.16}$ & $40.4$ & $29.3$\\
\bottomrule
\end{tabular}
\end{table}
\vspace{-0.8cm}
In Table~\ref{tab:calinfo}, all methods remain near nominal coverage, so their widths are directly comparable. As $k$ increases in the halo designs, more observations near the targets are assigned to calibration. GSCP and LSCP widen because neither can use those observations to predict the target. The whitened width stays nearly unchanged, since the total information available from the labelled observations remains constant in these designs. Whitening also gives the narrowest interval in the other three designs. Across the six designs, the width reduction of whitening relative to GSCP (\emph{Achieved gain}) has correlation $0.98$ with the available gain. Regime B demonstrates that the advantage of whitening increases as the calibration data become more informative about the target than the fitting set.

\paragraph{\emph{Regime C: coverage beyond random sampling.}}
Regime C distinguishes nonrandom spatial assignment from failures of calibration--target comparability. The upper row of Figure~\ref{fig:designC} assigns the fitting, calibration, and target observations to different parts of the domain under a stationary process (Table~\ref{tab:nonexch}). All methods maintain near-nominal coverage here, and whitening produces the narrowest interval. It is roughly $40\%$ narrower than naive split conformal and $6$–$39\%$ narrower than GSCP and LSCP.

\begin{figure}[t]
\centering
\includegraphics[width=\textwidth]{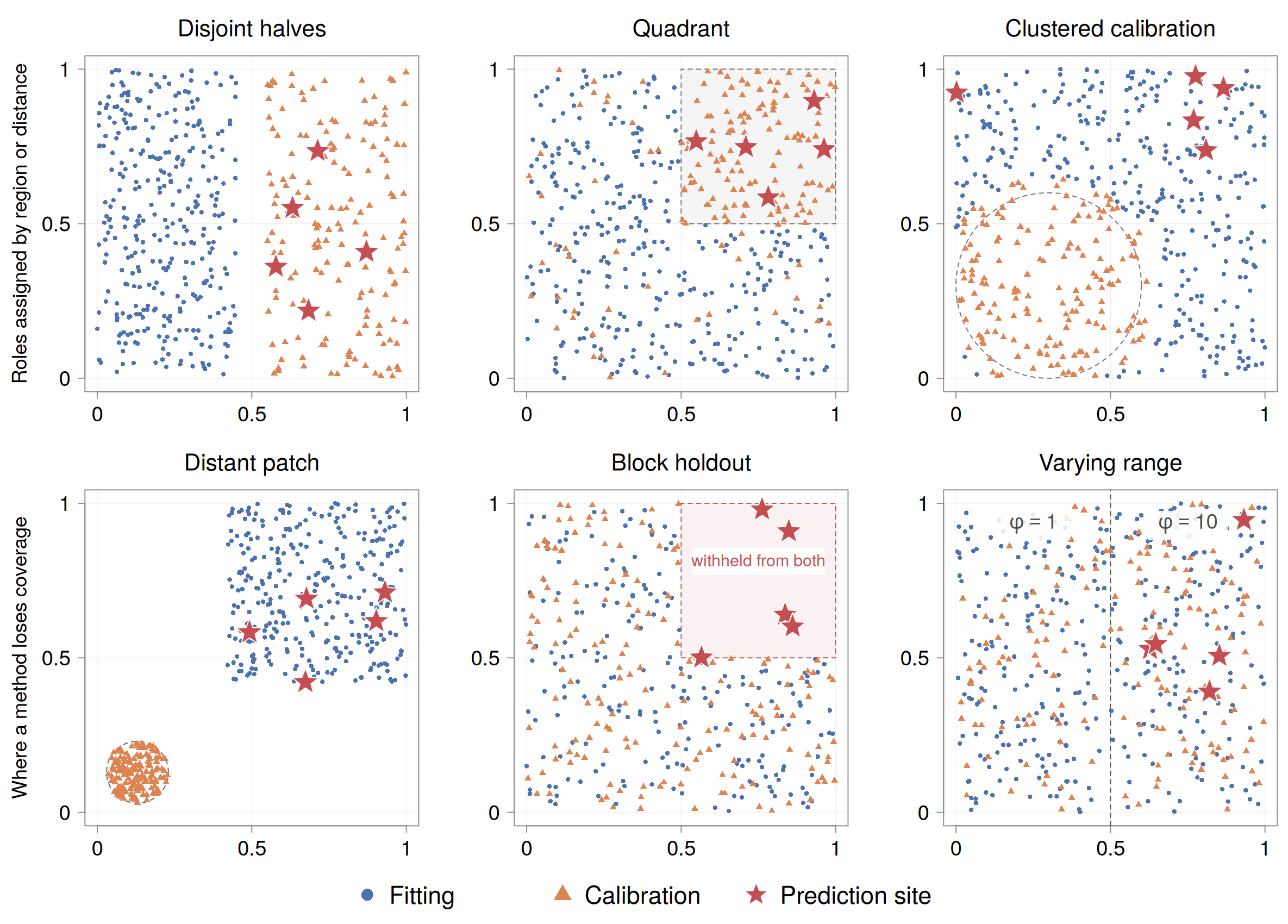}
\caption{Designs used to distinguish nonrandom spatial assignment from failures of calibration--target comparability. Symbols are as in Figure~\ref{fig:designB}. The upper row shows disjoint-halves, quadrant, and clustered-calibration designs under a stationary process. The lower row shows distant-patch calibration, block holdout, and varying dependence range.}
\label{fig:designC}
\end{figure}
\begin{table}[t]
\centering
\small
\setlength{\tabcolsep}{2.5pt}
\caption{Coverage beyond random sampling. Method entries report coverage/mean width. The last column gives the coverage predicted for the whitened interval from $200$ separate diagnostic observations in the target region. Under distant patch, LSCP is unbounded at $25\%$ of the targets, which we count as covered; its finite intervals have coverage/mean width $.876/2.14$. Nominal coverage is $0.90$.}
\label{tab:nonexch}
\begin{tabular}{lccccr} \toprule Design & Naive & GSCP & LSCP & Whitened & Predicted coverage\\ \midrule
\multicolumn{6}{l}{\emph{Role-dependent sampling}}\\
Disjoint halves       & $.897/3.41$ & $.900/3.22$ & $.910/3.24$ & $.914/\mathbf{1.98}$ & $.908$\\
Quadrant              & $.889/3.15$ & $.895/2.73$ & $.902/2.79$ & $.907/\mathbf{1.93}$ & $.905$\\
Clustered calibration & $.903/3.21$ & $.906/2.04$ & $.923/2.18$ & $.904/\mathbf{1.93}$ & $.901$\\
\midrule
\multicolumn{6}{l}{\emph{Coverage stress tests}}\\
Distant patch         & $.914/3.34$ & $.850/1.90$ & $.907/\infty$ & $.908/\mathbf{1.86}$ & $.911$\\
Block holdout         & $.853/3.05$ & $.870/2.74$ & $.880/2.89$ & $.873/\mathbf{2.64}$ & $.886$\\
Varying range         & $.866/3.39$ & $.876/2.39$ & $.902/2.59$ & $.878/\mathbf{2.29}$ & $.861$\\
\bottomrule\end{tabular}
\end{table}
\vspace{-0.95cm}
In the distant-patch design, the calibration observations are located in a small patch, far from the targets. GSCP undercovers here. LSCP does not solve this problem either, and reaches nominal coverage only when its unbounded intervals are counted as covering. The calibration errors in the patch are closely correlated with each other but nearly uncorrelated with the target error. Whitening explicitly accounts for this correlation, and hence attains nominal coverage with the narrowest interval.

Under block holdout the covariance family is correct, but the calibration observations have near neighbours in the fitting set and the targets do not. As a result, their predictive variances respond differently to error in the estimated range, which can leave the transformed target errors larger than the calibration scores. Every method undercovers here.

Under varying range, every method undercovers except LSCP. Localization downweights calibration observations far from the target. The two halves are separated by a straight line, so this prioritises calibration sites in the target's own half that share its dependence range. A stationary working covariance cannot capture the sudden jump in the range, which results in undercoverage for our proposed approach.

It is reasonable to wonder if such undercoverage can be predicted in the absence of knowledge about the actual data-generating process. Our proposed diagnostic in Section~\ref{sec:theory_loc} is developed to answer this. We evaluate the diagnostic on Regime C, which includes the designs where coverage fails. In each replicate, $\widehat\krel$ is estimated from $200$ additional labelled observations drawn from the region the targets occupy, separate from the $501$ sites used for fitting, calibration, and evaluation. The predicted coverage is $G(\xi_{1-\alpha}/\widehat\krel)$. Please refer to Supplementary Material Section~\ref{supp:impl_prosp} for implementation details.

In the four designs where whitening maintains coverage, the predicted and realized coverage differ by at most $0.006$. Under block holdout and varying range, the diagnostic correctly predicts undercoverage, within $2$ percentage points of the realized value. Regime C shows that the diagnostic separates the settings in which whitening remains reliable from those at risk of undercoverage, and correctly approximates the size of the loss.

\section{Fine particulate matter monitoring network}
\label{sec:data}

\subsection{Data and evaluation}
\label{sec:data_setup}
Air-quality assessments need estimates of annual PM$_{2.5}$ concentrations on a fine grid, but most grid cells have no monitor. Assessing exposure there requires predicting the concentration and knowing how reliable the prediction is. Environmental and geographic covariates are often used to predict the mean concentration, but their relationship with PM$_{2.5}$ is complex and unknown. Therefore, it is difficult to rely on inference from a parametric model, which often requires its assumptions to be right to guarantee the validity of the inference. Conformal prediction is a natural alternative; however, its naive form ignores the spatial correlation among nearby monitors. The monitoring network is thus precisely the kind of application for which spatial conformal inference is intended.

We analyse annual mean PM$_{2.5}$ concentrations in 2022 at $723$ Environmental Protection Agency monitors covering the contiguous United States. A random forest that uses environmental covariates has an out-of-bag $R^2$ of $0.23$. After fitting a spatial model to the residuals, we find that spatial dependence accounts for $67\%$ of the residual variance. The practical range of the spatial dependence is around $460$ kilometres. Nearby residuals are therefore likely to be informative about one another, as a typical monitor has around $69$ neighbouring monitors within that range. For details on data sources and preparation, please see Supplementary Material Section~\ref{supp:data}. 

To evaluate the methods, we hold out monitors and check if their intervals contain the recorded concentrations. The remaining monitors are split at random into fitting (two thirds) and calibration (one third) sets. We address two questions as follows:
\begin{enumerate}
    \item Does whitening result in narrower intervals at comparable coverage inside the monitoring network? We split the monitors into twenty \emph{random folds} and use each fold in turn as the test set.

    \item Does coverage degrade when predicting in a location where there are no observed data? If so, can the diagnostic detect the regions at risk? We use \emph{whole-state holdout}, where each of the $49$ state-level regions is held out in turn and its monitors predicted from observations outside the state.
\end{enumerate}

\subsection{Random folds}
\label{sec:data_results}
When we pool the results across the twenty folds, the four methods remain broadly close to nominal coverage (Table~\ref{tab:epa_overall}). The whitened method produces the narrowest interval, by approximately $6\%$ and $4\%$ against GSCP and LSCP, respectively, and $18\%$ against naive split conformal. The prediction performance may depend on the distance between a target monitor and the nearest fitting monitor. Thus, Figure~\ref{fig:epa_quint} groups target monitors into five quintiles according to their distances. The naive interval ignores location; thus, it is consistently wide across all five groups. The other three methods start about $26$ to $29\%$ narrower than naive in the nearest quintile and widen with distance. The whitened interval is the narrowest everywhere except in the nearest quintile, where GSCP and LSCP are narrower but undercover by seven to eight percentage points. The widths of the GSCP and LSCP intervals reach those of naive in the furthest quintile, but the whitened interval remains $8$ to $10\%$ below them.

\begin{table}[t]
\centering
\small
\caption{Coverage and interval width in the PM$_{2.5}$ application. Each entry reports coverage/mean width in $\mu\mathrm{g}/\mathrm{m}^3$. Random-fold results are pooled across the twenty folds. Whole-state results are computed for each of the $49$ regions and averaged with equal weight. Nominal coverage is $0.90$.}
\label{tab:epa_overall}
\begin{tabular}{lcccc}
\toprule
Scheme & Naive & GSCP & LSCP & Whitened \\
\midrule
Random folds        & $.903/5.94$ & $.887/5.16$ & $.892/5.07$ & $.898/4.85$ \\
Whole-state holdout & $.868/5.44$ & $.873/5.23$ & $.873/5.10$ & $.879/5.13$ \\
\bottomrule
\end{tabular}
\end{table}
\vspace{-0.45cm}

\begin{figure}[t]
\centering
\includegraphics[width=\textwidth]{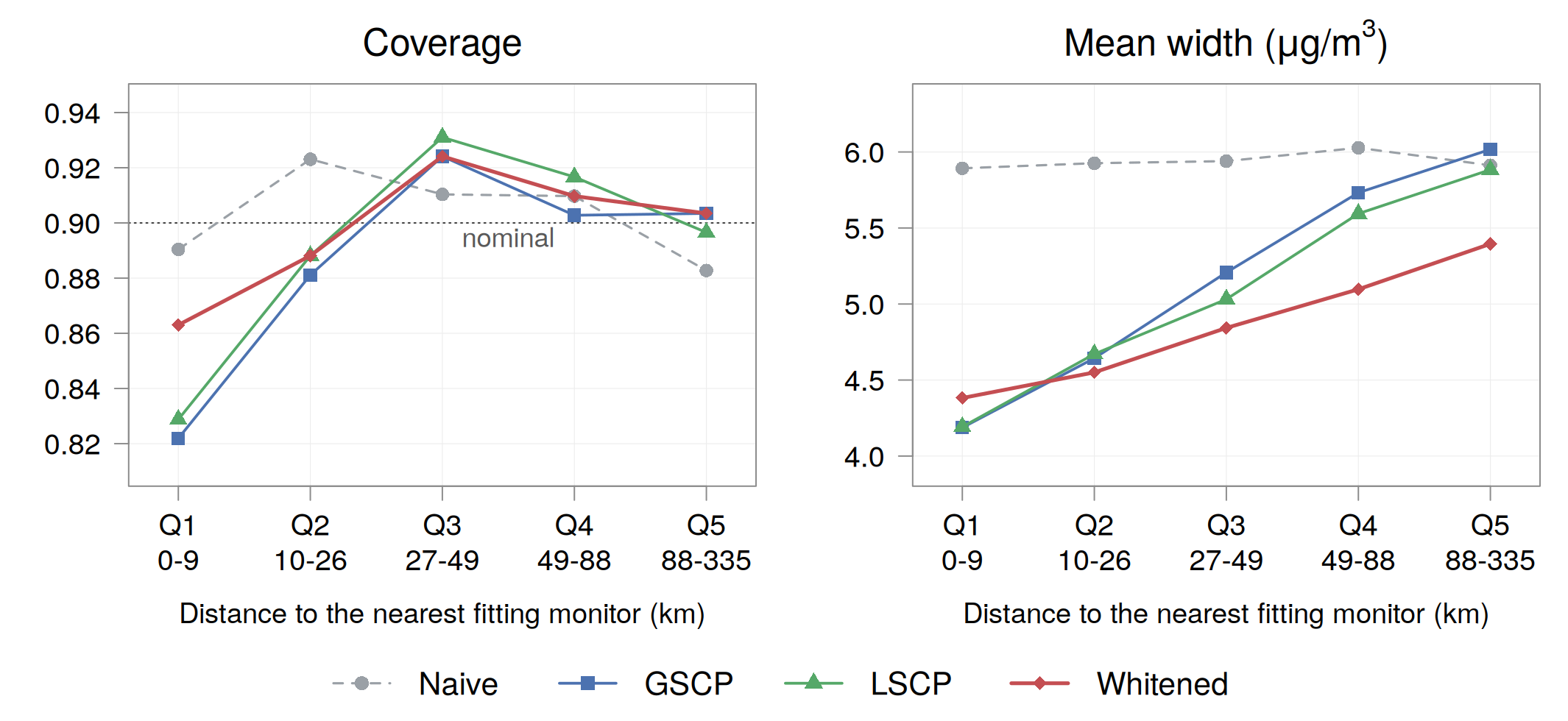}
\caption{Empirical coverage and mean width by distance to the nearest fitting monitor under random folds at nominal coverage $0.90$. GSCP and LSCP reach the naive width as distance increases. Their smaller widths in the nearest quintile are accompanied by coverage seven to eight percentage points below nominal.}
\label{fig:epa_quint}
\end{figure}
\vspace{-0.45cm}
\subsection{Whole-state holdout}
\label{sec:data_state}
Whole-state holdout is much more difficult because each target monitor is in a region with no fitting or calibration observations. As a result, all four approaches fall short of nominal coverage. LSCP and the whitened interval are the two best-performing methods here, and are practically indistinguishable in both coverage and width. The general undercoverage is expected, since here the calibration scores are derived from regions that do not always mirror the one being predicted. The most pertinent question is whether this failure is predictable in each region, which is what the diagnostic of Section~\ref{sec:theory_loc} answers. 

We test this prospectively by randomly separating the monitors in each held-out region into a diagnostic half and a disjoint evaluation half. We use the responses in the diagnostic half to estimate the scale, and the evaluation half to measure the coverage. We repeat the split $20$ times to reduce sensitivity to how the monitors in smaller regions are divided. To ensure there are enough monitors in both halves, we restrict the analysis to regions with at least sixteen monitors, which leaves us with $19$ regions.

When we average over the splits, $\widehat\krel$ has Spearman correlation $-0.90$ with realized coverage across the $19$ regions. We group the regions into terciles of $\widehat\krel$ (Table~\ref{tab:epa_diag}). Since a higher value of $\widehat\krel$ indicates undercoverage, the highest tercile of $\widehat\krel$ corresponds to an estimated coverage of $0.769$, which increases to $0.987$ in the lowest tercile. Realized coverage closely follows this trend and is within one percentage point of the estimate in the lower two terciles. In the highest, the estimate overstates it by about two and a half points. 

\begin{table}[t]
\centering
\small
\caption{Regional diagnostic under whole-state holdout. Regions are grouped into terciles of $\widehat\krel$, estimated from a diagnostic half of each region's monitors. Coverage is measured on the disjoint evaluation half, and estimated coverage is $G(\xi_{1-\alpha}/\widehat\krel)$. Estimated and realized coverage are averaged over $20$ random splits within each region and then over regions with equal weight, and regions with fewer than sixteen monitors are excluded. Nominal coverage is $0.90$.}
\label{tab:epa_diag}
\begin{tabular}{lcccc} \toprule Tercile of $\widehat\krel$ & Regions & Mean $\widehat\krel$ & Estimated coverage & Realized coverage \\ \midrule Lowest & $7$ & $0.60$ & $0.987$ & $0.979$ \\ Middle & $6$ & $0.93$ & $0.915$ & $0.909$ \\ Highest & $6$ & $1.38$ & $0.769$ & $0.744$ \\ \bottomrule \end{tabular}
\end{table}
\vspace{-0.5cm}
\section{Discussion}
\label{sec:disc}
Existing spatial conformal methods do not use the predictive power of the calibration residuals, which leaves shared spatial information in both the calibration and the target residuals. Conditioning on them directly is not possible, since each calibration residual would then be conditioned on itself, resulting in a degenerate score. We propose to bypass this problem by conditioning on the calibration residuals sequentially. When the residuals are elliptical and the working covariance is correct, this restores score exchangeability. Under additional conditions, it also stabilizes the threshold and produces an interval asymptotically as short as the oracle.

When the working covariance captures the residual dependence sufficiently well, the coverage is nominal and the whitened interval is shorter and more stable than alternatives. The advantage of this method is largest when the observations near the target are in the calibration set, which cannot be used for prediction in GSCP and LSCP.  A similar pattern holds for the air-quality application. Under random folds, whitening produced the narrowest interval, with its margin over GSCP and LSCP largest at monitors far from the fitting set. Under whole-state holdout, it was comparable to LSCP.

Since we do not know the true data-generating process in practice, the working covariance may not be an adequate representation of the residual dependence structure. As a result, all methods that use a working covariance matrix, including GSCP, LSCP and our whitening approach, might suffer from undercoverage. We develop a diagnostic for our method that identifies the regions at risk without knowing the true data-generating process. In the simulations it captured the coverage loss, and in the application it identified the high-risk regions.

Although not entirely distribution-free, the approach is model-lean. An elliptical residual law and a correct working covariance are sufficient for exact coverage, and further conditions are needed for efficiency and local-coverage outcomes. Additionally, the estimated covariance is treated as fixed in the efficiency and misspecification results. Future research should naturally focus on building a diagnostic-based adjustment with a formal coverage guarantee and extending the guarantees to approximated and nonstationary covariance models. 

\section*{Code availability}
Code to reproduce all numerical results is available at \url{https://github.com/ArkajyotiSaha/spatial-conformal-whitening}.

\section*{Acknowledgements}
We acknowledge the use of generative AI to help polish the manuscript and write code for the numerical experiments. The authors independently reviewed all work and take sole responsibility for the content.

\bibliographystyle{apalike}
\bibliography{references}

\clearpage
\normalsize

\setcounter{page}{1}
\renewcommand{\thepage}{S\arabic{page}}

\setcounter{section}{0}
\renewcommand{\thesection}{S\arabic{section}}

\setcounter{subsection}{0}
\renewcommand{\thesubsection}{\thesection.\arabic{subsection}}

\setcounter{subsubsection}{0}
\renewcommand{\thesubsubsection}{\thesubsection.\arabic{subsubsection}}

\setcounter{equation}{0}
\renewcommand{\theequation}{S\arabic{equation}}

\setcounter{theorem}{0}
\renewcommand{\thetheorem}{S\arabic{theorem}}

\setcounter{lemma}{0}
\renewcommand{\thelemma}{S\arabic{lemma}}

\setcounter{proposition}{0}
\renewcommand{\theproposition}{S\arabic{proposition}}

\setcounter{assumption}{0}
\renewcommand{\theassumption}{S\arabic{assumption}}
\renewcommand{\theHsection}{S\arabic{section}}
\renewcommand{\theHequation}{S\arabic{equation}}
\renewcommand{\theHtheorem}{S\arabic{theorem}}
\renewcommand{\theHlemma}{S\arabic{lemma}}
\renewcommand{\theHproposition}{S\arabic{proposition}}
\renewcommand{\theHassumption}{S\arabic{assumption}}
\setcounter{table}{0}\renewcommand{\thetable}{S\arabic{table}}\renewcommand{\theHtable}{S\arabic{table}}
\setcounter{figure}{0}\renewcommand{\thefigure}{S\arabic{figure}}\renewcommand{\theHfigure}{S\arabic{figure}}

\newcommand{\Fbar}{\overline{F}}
\newcommand{\bU}{\mathbf{U}}
\newcommand{\TV}{\operatorname{TV}}
\newcommand{\tr}{\operatorname{tr}}
\providecommand{\bSigma}{\boldsymbol{\Sigma}}
\providecommand{\operator}[1]{\operatorname{#1}}
\begin{center}
{\Large\bfseries Supplementary Material for\par}
\vspace{0.5em}
{\Large\bfseries ``Conformal Prediction for Spatially Dependent Data via Sequential Whitening''\par}
\end{center}

\vspace{1em}

\noindent This Supplementary Material contains the technical proofs, implementation details, and supporting numerical results for the main text. Sections~\ref{supp:sec_whitening}--\ref{supp:sec_eff} develop the exact-whitening theory, including finite-sample validity and oracle efficiency. Sections~\ref{supp:sec_gap} and~\ref{supp:sec_loc} study imperfect whitening and develop the prospective relative-scale diagnostic. Section~\ref{supp:sec_impl} contains details of the simulation designs and competing methods, and examines sensitivity to the main implementation choices. Section~\ref{supp:data} then describes the details of data processing steps for the PM$_{2.5}$  application.

\section{Sequential prediction is triangular whitening}
\label{supp:sec_whitening}

We prove Lemma~\ref{lem:sequential_whitening} here. This links the implemented sequential predictors to the matrix transformation used in the following arguments.

\begin{proof}[Proof of Lemma~\ref{lem:sequential_whitening}]
Fix $i$ and recall $k_i=n_{\mathcal F}+i$ and $I_i=\{1,\ldots,k_i-1\}$. Let
$\bell=(\ell_{k_i1},\ldots,\ell_{k_i,k_i-1})^{\top}$ be the first $k_i-1$ entries of the $k_i$th row of
$\bL$. The leading block of $\bL$ and the corresponding blocks of
$\bSigma_+^{\mathrm w}=\bL\bL^{\top}$ satisfy
\begin{equation}
\label{eq:supp_chol_blocks}
\bL_{1:k_i,1:k_i}
=
\begin{pmatrix}
\bL_{I_i,I_i}&\mathbf 0\\
\bell^{\top}&\ell_{k_ik_i}
\end{pmatrix},
\quad
\bSigma_{I_i,I_i}^{\mathrm w}=\bL_{I_i,I_i}\bL_{I_i,I_i}^{\top},
\quad
\bSigma_{k_i,I_i}^{\mathrm w}=\bell^{\top}\bL_{I_i,I_i}^{\top},
\quad
\bSigma_{k_i,k_i}^{\mathrm w}=\bell^{\top}\bell+\ell_{k_ik_i}^2.
\end{equation}
Substituting into \eqref{eq:seq_musd} gives
$\widehat\mu_i=\bell^{\top}\bL_{I_i,I_i}^{-1}\br_{+,I_i}$ and $\widehat\nu_i=\ell_{k_ik_i}^2$. We can then use forward
substitution in \eqref{eq:supp_chol_blocks} to get 
\[
\big(\bL^{-1}\br_+\big)_{k_i}
=
\frac{r_i-\widehat\mu_i}{\widehat\nu_i^{1/2}}
=e_i,
\]
which is \eqref{eq:sequential_whitening}.
\end{proof}

\section{Exact coverage after whitening}
\label{supp:sec_exact}

Theorem~\ref{thm:whitened_interval} demonstrates that the same transformation that enables us to utilise calibration residuals to predict the target also recovers the exchangeability needed for the conformal coverage guarantee. 

\begin{proof}[Proof of Theorem~\ref{thm:whitened_interval}]
$\br_+$ is centred and elliptical with scatter matrix $\bSigma_+^{\mathrm w}=\bL\bL^{\top}$. Hence, we can write it as follows:
\begin{equation}
\label{eq:supp_elliptical}
\br_+\overset{d}{=}R\bL\bU,
\end{equation}
where $\bU$ is uniform on the unit sphere $\mathbb S^{N-1}$ and $R>0$ is independent of $\bU$ \citep[Theorem 2.5]{fang1990}. If $\E R^2<\infty$, the covariance of $\br_+$ is $(\E R^2/N)\bSigma_+^{\mathrm w}$. In the Gaussian case with covariance $\bSigma_+^{\mathrm w}$, $R^2\sim\chi^2_N$.

\emph{Step 1: whitening produces a spherical vector.} Applying $\bL^{-1}$ to \eqref{eq:supp_elliptical} gives
\begin{equation}
\label{eq:supp_spherical}
\be_+
=
\bL^{-1}\br_+
\overset{d}{=}
R\bU.
\end{equation}
For every orthogonal $N\times N$ matrix $\bm Q$, rotational invariance of the uniform distribution on the sphere gives $\bm Q\bU\overset{d}{=}\bU$. Independence of $R$ and $\bU$ therefore implies $\bm Q\be_+\overset{d}{=}\be_+$.

\emph{Step 2: the transformed calibration and target errors are exchangeable.} Let $\bm Q_\pi$ permute the last $m+1$ coordinates and fix the initial $n_{\mathcal F}$ coordinates. Since $\bm Q_\pi$ is orthogonal, \eqref{eq:supp_spherical} implies $\bm Q_\pi\be_+\overset{d}{=}\be_+$. Lemma~\ref{lem:sequential_whitening} shows that the last $m+1$ coordinates are $(e_1,\ldots,e_{m+1})$, which proves that these errors, and thus their absolute values, are exchangeable.

\emph{Step 3: coverage is a rank event.} Lemma~\ref{lem:sequential_whitening} at $i=m+1$ gives
\begin{equation}
\label{eq:supp_cover_iff}
y_0\in\widehat C_\alpha^{\mathrm W}(\bx_0,\bs_0)
\quad\Longleftrightarrow\quad
|e_{m+1}|\le\widehat q^{\,\mathrm W}_{1-\alpha}.
\end{equation}
The law of $\be$ is absolutely continuous on $\R^{m+1}$, hence absolute transformed errors have no ties almost surely. Therefore, their exchangeability makes the rank of $|e_{m+1}|$ among $|e_1|,\ldots,|e_{m+1}|$ uniform on $\{1,\ldots,m+1\}$. By construction, the event in \eqref{eq:supp_cover_iff} occurs exactly when this rank is at most $j_\alpha$. Consequently,
\begin{equation}
\label{eq:supp_exact_coverage}
\Prob\{y_0\in\widehat C_\alpha^{\mathrm W}(\bx_0,\bs_0)\}
=
\frac{j_\alpha}{m+1}
=
\frac{\lceil(1-\alpha)(m+1)\rceil}{m+1}
\ge 1-\alpha.
\end{equation}
When $j_\alpha=m+1$, this probability becomes one and it results in an unbounded interval.
\end{proof}

\section{Interval width and asymptotic efficiency}
\label{supp:sec_eff}

The whitened and GSCP intervals differ in two ways. Whitening replaces the GSCP calibration quantile with a quantile of the whitened scores. It also replaces the predictive variance given the fitting set with the predictive variance given the full labelled sample. In Section~\ref{supp:proof_width}, we prove that the predictive variances determine the asymptotic width ratio of the whitened and GSCP intervals. Section~\ref{supp:sec_oracle_eff} then proves Theorem~\ref{thm:eff}, which shows that the proposed whitened interval is asymptotically as short as the shortest conditionally valid interval based on the full labelled sample.

\subsection{Proof of the width ratio}
\label{supp:proof_width}

\begin{proof}[Proof of Proposition~\ref{prop:width}]
The GSCP and whitened widths satisfy
\begin{equation}
\label{eq:supp_width_ratio}
\frac{W^{\mathrm W}(\bs_0)}{W^{\mathrm G}(\bs_0)}
=
\frac{\widehat q^{\,\mathrm W}_{1-\alpha}}
     {\widehat q^{\,\mathrm G}_{1-\alpha}}
\left\{
\frac{\widehat\nu_{\mathcal F\cup\mathcal C}(\bs_0)}
     {\widehat\nu_{\mathcal F}(\bs_0)}
\right\}^{1/2}.
\end{equation}
By the assumption of Proposition~\ref{prop:width}, the empirical distributions of the two sets of absolute calibration scores converge in probability to a common limit $G_0$, which is continuous and strictly increasing at $\zeta_{1-\alpha}=G_0^{-1}(1-\alpha)>0$. Hence
\[
\widehat q^{\,\mathrm W}_{1-\alpha}\overset{p}{\longrightarrow}\zeta_{1-\alpha},
\qquad
\widehat q^{\,\mathrm G}_{1-\alpha}\overset{p}{\longrightarrow}\zeta_{1-\alpha}.
\]
Their ratio is therefore $1+o_{\Prob}(1)$. Substitution into \eqref{eq:supp_width_ratio} proves \eqref{eq:width_ratio}.
\end{proof}

\subsection{Proof of asymptotic oracle efficiency}
\label{supp:sec_oracle_eff}

Let $\mathcal D=\mathcal F\cup\mathcal C$. The coordinates of $\be_+=\bL^{-1}\br_+$ are i.i.d.\ with distribution $F$ according to the assumptions of Theorem~\ref{thm:eff}. Here $F$ has mean zero and unit variance, and its density $f_e$ is continuous, symmetric about zero, and strictly decreasing on $(0,\infty)$. Let us define
\begin{equation}
\label{eq:supp_H}
H(t)=\Prob_F(|e|\le t),
\qquad
\xi^H_{1-\alpha}=H^{-1}(1-\alpha),
\end{equation}
and assume the density of $H$ is positive and continuous at $\xi^H_{1-\alpha}$. Next, we prove Theorem~\ref{thm:eff} in two steps. First, Lemma~\ref{supp:lem_oracle} gives a closed form for the oracle width in terms of $H$, and then we prove that the whitened interval attains that width asymptotically.

\begin{lemma}[Form of the oracle interval]
\label{supp:lem_oracle}
Under the assumptions of Theorem~\ref{thm:eff}, the shortest interval with conditional coverage at least $1-\alpha$ given $\br_{\mathcal D}$ is
\begin{equation}
\label{eq:supp_oracle}
I^{\mathrm o}(\bs_0)
=
\widehat f(\bx_0)
+\E(r_0\mid\br_{\mathcal D})
\pm
\xi^H_{1-\alpha}\nu(\bs_0)^{1/2},
\qquad
\nu(\bs_0)=\Var(r_0\mid\br_{\mathcal D}),
\end{equation}
and hence $W^{\mathrm o}(\bs_0)=2\xi^H_{1-\alpha}\nu(\bs_0)^{1/2}$.
\end{lemma}

\begin{proof}
Because $\bL$ is lower triangular and invertible, the first $N-1$ coordinates of $\be_+$ are a measurable bijection of $\br_{\mathcal D}$. The final row of $\br_+=\bL\be_+$ gives
\begin{equation}
\label{eq:supp_target_innovation}
r_0
=
\sum_{j=1}^{N-1}\ell_{Nj}(\be_+)_j
+\ell_{NN}(\be_+)_N.
\end{equation}
The first term is measurable with respect to $\br_{\mathcal D}$, while $(\be_+)_N=e_{m+1}$ is independent of $\br_{\mathcal D}$ with distribution $F$. Since $F$ has mean zero and unit variance,
independence gives $\E\{(\be_+)_N\mid\br_{\mathcal D}\}=0$ and
$\Var\{(\be_+)_N\mid\br_{\mathcal D}\}=1$, so that
\[
\E(r_0\mid\br_{\mathcal D})
=
\sum_{j=1}^{N-1}\ell_{Nj}(\be_+)_j,
\qquad
\nu(\bs_0)=\ell_{NN}^2.
\]
Using the symmetry of the density $f_e$, and the fact that it is strictly decreasing away from zero, we have that $[-\xi^H_{1-\alpha},\xi^H_{1-\alpha}]$ is the shortest interval with probability at least $1-\alpha$. The oracle intervals in \eqref{eq:supp_oracle} are then obtained by shifting and scaling this interval by the conditional location and scale, followed by the addition of the estimated mean $\widehat f(\bx_0)$.
\end{proof}

\begin{proof}[Proof of Theorem~\ref{thm:eff}]
Lemma~\ref{lem:sequential_whitening} and \eqref{eq:supp_target_innovation} show that the working predictive variance used by the whitened interval is $\widehat\nu_{\mathcal F\cup\mathcal C}(\bs_0)=\ell_{NN}^2=\nu(\bs_0)$. Thus
\begin{equation}
\label{eq:supp_eff_ratio}
\frac{W^{\mathrm W}(\bs_0)}{W^{\mathrm o}(\bs_0)}
=
\frac{\widehat q^{\,\mathrm W}_{1-\alpha}}
     {\xi^H_{1-\alpha}}.
\end{equation}
By \eqref{eq:supp_eff_ratio} the theorem reduces to showing that $\widehat q^{\,\mathrm W}_{1-\alpha}/\xi^H_{1-\alpha}=1+O_{\Prob}(m^{-1/2})$, and since $\xi^H_{1-\alpha}>0$ is fixed, this is equivalent to proving  $\widehat q^{\,\mathrm W}_{1-\alpha}-\xi^H_{1-\alpha}=O_{\Prob}(m^{-1/2})$. We note that for $p_m=j_\alpha/m$, $\widehat q^{\,\mathrm W}_{1-\alpha}-\xi^H_{1-\alpha} = \left(\widehat q^{\,\mathrm W}_{1-\alpha}-H^{-1}(p_m)\right) + \left( H^{-1}(p_m)-\xi^H_{1-\alpha}\right)$, and prove the result by controlling each term separately.

First, note that the transformed calibration scores $|e_1|,\ldots,|e_m|$ are independent with distribution function $H$, and $j_\alpha\le m$ for all sufficiently large $m$, so the adjoined value $+\infty$ does not affect the $j_\alpha$th order statistic. Hence $\widehat q^{\,\mathrm W}_{1-\alpha}$ is the empirical quantile of the sample at level $p_m=j_\alpha/m$. Since $p_m=1-\alpha+O(m^{-1})$ and $H$ has a density that is positive and continuous at $\xi^H_{1-\alpha}$, asymptotic normality of central order statistics \citep[Section~2.3.3]{serfling1980} implies that $\widehat q^{\,\mathrm W}_{1-\alpha}-H^{-1}(p_m)=O_{\Prob}(m^{-1/2})$.

Next, notice that the conformal rank $j_\alpha$ must be an integer. This implies that the level $p_m$ differs from $1-\alpha$ by at most order $1/m$. Since $H^{-1}$ is differentiable at $1-\alpha$, it follows that $H^{-1}(p_m)-\xi^H_{1-\alpha}=O(m^{-1})$. Combining both results, we have $\widehat q^{\,\mathrm W}_{1-\alpha}-\xi^H_{1-\alpha} = O_{\Prob}(m^{-1/2}) + O(m^{-1})$. Since the first quantity dominates the second, we have $\widehat q^{\,\mathrm W}_{1-\alpha}-\xi^H_{1-\alpha}=O_{\Prob}(m^{-1/2})$, which proves \eqref{eq:eff}.
\end{proof}

\section{Coverage when whitening is imperfect}
\label{supp:sec_gap}

Next, we consider what happens when the whitening covariance is not correct.

Let $P_{\be}$ be the actual law of $\be=(e_1,\ldots,e_m,e_{m+1})^{\top}$ and let $Q_{\be}$ be a reference law under which $|e_1|,\ldots,|e_{m+1}|$ are exchangeable and almost surely distinct.

\begin{proof}[Proof of Theorem~\ref{thm:gap}]
By \eqref{eq:supp_cover_iff}, coverage is given by the event
\[
E=\{|e_{m+1}|\le\widehat q^{\,\mathrm W}_{1-\alpha}\},
\]
which depends only on $\be$. Under $Q_{\be}$, the rank argument from the proof of Theorem~\ref{thm:whitened_interval} gives
\begin{equation}
\label{eq:supp_Qcoverage}
\big|Q_{\be}(E)-(1-\alpha)\big|\le\frac{1}{m+1}.
\end{equation}
Total variation distance is the supremum of $|P(A)-Q(A)|$ over all measurable $A$. Hence we have the following inequality: $|P_{\be}(E)-Q_{\be}(E)|\le\TV(P_{\be},Q_{\be})$. 
Combining this inequality with \eqref{eq:supp_Qcoverage} we have
\begin{equation}
\label{eq:supp_gap_tv}
\big|P_{\be}(E)-(1-\alpha)\big|
\le
\frac{1}{m+1}+\TV(P_{\be},Q_{\be}).
\end{equation}
Pinsker's inequality \citep[Lemma~2.5]{tsybakov2009} tells us that total variation distance is bounded above using the KL divergence as follows:
\[
\TV(P_{\be},Q_{\be})
\le
\left\{\frac{\KL(P_{\be}\,\|\,Q_{\be})}{2}\right\}^{1/2}.
\]
Combining this with \eqref{eq:supp_gap_tv} gives us \eqref{eq:coverage_gap_kl}.
\end{proof}

\section{Relative scale and the prospective coverage diagnostic}
\label{supp:sec_loc}

Theorem~\ref{thm:gap} bounds the coverage error, but to evaluate the bound, we need to know the true covariance. However, this information is not available in real data applications. Theorem~\ref{thm:loc} bypasses this problem. It gives a way to identify regions at risk of undercoverage that can be estimated from residuals alone.

Throughout this section, we assume the residual vector to be centred Gaussian with positive-definite covariance, so the transformed errors $e_1,\ldots,e_m$ and $z_0=e_{m+1}$ are jointly Gaussian and nondegenerate, with $e_i\sim N(0,\kappa_i^2)$ and $z_0\sim N(0,\kappa_0^2)$, where $\kappa_0=\kappa(\bs_0)$. Let $G(t)=2\Phi(t)-1$ and $g(t)=2\phi(t)$ denote the distribution function and density of a standard half-normal variable, and write $\xi_{1-\alpha}=G^{-1}(1-\alpha)$. Let $\bm R_m$ be the correlation matrix of $(e_1,\ldots,e_m)$ and
\begin{equation}
\label{eq:supp_vm}
v_m=\frac{\|\bm R_m\|_F^2}{m^2}=\frac1{m^2}\sum_{i,j=1}^m\Corr(e_i,e_j)^2 .
\end{equation}

Coverage is the event that $|z_0|$ falls below the conformal threshold. The proof has three steps. First, we show that the threshold converges to $q_m^\star$. This uses the bounds on the calibration scales and requires $v_m\to0$. Second, coverage at $q_m^\star$ equals $G\{\xi_{1-\alpha}/\krel(\bs_0)\}$. Third, we bound the error from replacing the random threshold by $q_m^\star$. This uses the lower bound on $\kappa(\bs_0)$. We make the following assumptions:

\begin{assumption}[Bounded calibration scales]
\label{supp:ass_scale}
There are constants $0<\kappa_{\min}\le\kappa_{\max}<\infty$ such that
$\kappa_{\min}\le\kappa_i\le\kappa_{\max}$ for all $i$ and $m$.
\end{assumption}

\begin{assumption}[Target scale]
\label{supp:ass_target}
There is a constant $\kappa_{0,\min}>0$ such that $\kappa_0\ge\kappa_{0,\min}$ along the asymptotic
sequence.
\end{assumption}

\begin{assumption}[Weak correlation after whitening]
\label{supp:ass_corr}
$v_m\to0$ as $m\to\infty$.
\end{assumption}

Under a correct Gaussian working covariance, $\bm R_m=\bI$, so $v_m=1/m$, and $\kappa_i=\kappa_0=1$; all three assumptions then hold. Next, we define the empirical distribution function of the absolute calibration errors ($F_m(t)$), its expectation ($\Fbar_m(t)$), and the fixed value the threshold targets ($q_m^\star$) as follows
\begin{equation}
\label{eq:supp_edf}
F_m(t)=\frac1m\sum_{i=1}^m\bone\{|e_i|\le t\},
\qquad
\Fbar_m(t)=\E F_m(t)=\frac1m\sum_{i=1}^mG\!\left(\frac{t}{\kappa_i}\right),
\qquad
q_m^\star=\Fbar_m^{-1}(1-\alpha),
\end{equation}
where the second equality uses the fact that $|e_i|/\kappa_i$ is standard half-normal.

Below we first prove that the conformal threshold converges to $q_m^\star$. We do so with the help of three lemmas. Lemma~\ref{supp:lem_cov} bounds the covariance between the indicators of two correlated Gaussian variables. Lemma~\ref{supp:lem_var} uses it to bound the variance of $F_m(t)$ by $v_m$. Lemma~\ref{supp:lem_quant} turns that variance bound into concentration of the threshold around $q_m^\star$.

\begin{lemma}[Correlated absolute-value indicators]
\label{supp:lem_cov}
Let $(X,Y)$ be centred and jointly Gaussian with positive variances and correlation $\rho$. Then for
all $a,b\ge0$,
\[
\bigl|\Cov\bigl(\bone\{|X|\le a\},\bone\{|Y|\le b\}\bigr)\bigr|\le\frac{\rho^2}{4}.
\]
\end{lemma}

\begin{proof}
We write $X=\sigma_XU$ and $Y=\sigma_YV$, where $(U,V)$ is bivariate normal, with $\Corr(U,V)=\rho$. We denote the indicators as $f(U)$ and $h(V)$, where $f(u)=\bone\{|u|\le a/\sigma_X\}$ and $h(v)=\bone\{|v|\le b/\sigma_Y\}$. Both of these are square integrable under the standard normal law. Hence, we can expand them in the orthonormal Hermite polynomials as $f=\sum_{k\ge0}c_kH_k$ and $h=\sum_{k\ge0}d_kH_k$. Next, we apply Mehler's formula, which gives $\E\{H_j(U)H_k(V)\}=\rho^k\bone\{j=k\}$. This implies $\Cov\{f(U),h(V)\}=\sum_{k\ge1}c_kd_k\rho^k$. Since both indicators are even, we have $c_k=d_k=0$ for odd $k$, so every remaining term has $k\ge2$ and $|\rho|^k\le\rho^2$. Next, we apply the Cauchy--Schwarz inequality to the coefficient sequences and obtain
\[
\bigl|\Cov\{f(U),h(V)\}\bigr|
\le\rho^2\Bigl(\sum_{k\ge1}c_k^2\Bigr)^{1/2}\Bigl(\sum_{k\ge1}d_k^2\Bigr)^{1/2}
=\rho^2\bigl[\Var\{f(U)\}\Var\{h(V)\}\bigr]^{1/2}
\le\frac{\rho^2}{4}.
\]
Here the last inequality uses the fact that an indicator has variance at most $1/4$.
\end{proof}

\begin{lemma}[Variance of the calibration distribution]
\label{supp:lem_var}
For every $t\ge0$, $\Var\{F_m(t)\}\le v_m/4$.
\end{lemma}

\begin{proof}
By Lemma~\ref{supp:lem_cov},
\[
\Var\{F_m(t)\}
=\frac1{m^2}\sum_{i,j=1}^m\Cov\bigl(\bone\{|e_i|\le t\},\bone\{|e_j|\le t\}\bigr)
\le\frac1{4m^2}\sum_{i,j=1}^m\Corr(e_i,e_j)^2=\frac{v_m}{4}.
\qedhere
\]
\end{proof}

\begin{lemma}[Concentration of the whitened conformal threshold]
\label{supp:lem_quant}
Under Assumption~\ref{supp:ass_scale}, let $c=\kappa_{\max}^{-1}g(2\kappa_{\max}\xi_{1-\alpha}/\kappa_{\min})$ and $\delta_0=\kappa_{\min}\xi_{1-\alpha}/2$. For all sufficiently large $m$ and all $\delta$ with $4/(cm)\le\delta\le\delta_0$,
\begin{equation}
\label{eq:supp_quant}
\Prob\bigl(|\widehat q^{\,\mathrm W}_{1-\alpha}-q_m^\star|>\delta\bigr)\le\frac{2v_m}{c^2\delta^2}.
\end{equation}
\end{lemma}

\begin{proof}
Since $\kappa_{\min}\le\kappa_i\le\kappa_{\max}$, we have $\Fbar_m(\kappa_{\min}\xi_{1-\alpha})\le1-\alpha\le \Fbar_m(\kappa_{\max}\xi_{1-\alpha})$. This implies $q_m^\star\in[\kappa_{\min}\xi_{1-\alpha},\kappa_{\max}\xi_{1-\alpha}]$ and $[q_m^\star-\delta_0,q_m^\star+\delta_0]\subseteq[\kappa_{\min}\xi_{1-\alpha}/2,\,2\kappa_{\max}\xi_{1-\alpha}]$. Since $g$ is decreasing on $[0,\infty)$, each summand $\kappa_i^{-1}g(t/\kappa_i)$ of $\Fbar_m'(t)=m^{-1}\sum_i\kappa_i^{-1}g(t/\kappa_i)$ is at least $c$ on that interval. Hence, for $0<\delta\le\delta_0$,
\begin{equation}
\label{eq:supp_Fbar_slope}
\Fbar_m(q_m^\star+\delta)\ge1-\alpha+c\delta,
\qquad
\Fbar_m(q_m^\star-\delta)\le1-\alpha-c\delta.
\end{equation}
For fixed $\alpha$, $j_\alpha\le m$ for all sufficiently large $m$, so the adjoined value $+\infty$ does not affect the $j_\alpha$th order statistic. The absolute values of the calibration scores have no ties almost surely, so $F_m(\widehat q^{\,\mathrm W}_{1-\alpha})=j_\alpha/m$, and $1-\alpha\le j_\alpha/m<1-\alpha+2/m$.

If $\widehat q^{\,\mathrm W}_{1-\alpha}>q_m^\star+\delta$, fewer than $j_\alpha$ scores lie at or below $q_m^\star+\delta$, so $F_m(q_m^\star+\delta)\le(j_\alpha-1)/m<1-\alpha+2/m$. By \eqref{eq:supp_Fbar_slope}, $F_m(q_m^\star+\delta)-\Fbar_m(q_m^\star+\delta)<2/m-c\delta\le-c\delta/2$ whenever $\delta\ge4/(cm)$. If $\widehat q^{\,\mathrm W}_{1-\alpha}<q_m^\star-\delta$, then $F_m(q_m^\star-\delta)\ge j_\alpha/m\ge1-\alpha$, so $F_m(q_m^\star-\delta)-\Fbar_m(q_m^\star-\delta)\ge c\delta$. We can bound both probabilities using Chebyshev's inequality and Lemma~\ref{supp:lem_var}. The first one is bounded by $(v_m/4)/(c\delta/2)^2=v_m/(c^2\delta^2)$ and the second by $v_m/(4c^2\delta^2)$. Combining them, we have \eqref{eq:supp_quant}.
\end{proof}

Lemma~\ref{supp:lem_quant} shows that the conformal threshold is close to $q_m^\star$ with high probability. We now use this to prove Theorem~\ref{thm:loc}.

\begin{proof}[Proof of Theorem~\ref{thm:loc}]
By \eqref{eq:supp_cover_iff}, coverage is the event $\{|z_0|\le\widehat q^{\,\mathrm W}_{1-\alpha}\}$. Since $z_0\sim N(0,\kappa_0^2)$,
\begin{equation}
\label{eq:supp_leading}
\Prob(|z_0|\le q_m^\star)
=G\!\left(\frac{q_m^\star}{\kappa_0}\right)
=G\!\left\{\frac{\xi_{1-\alpha}}{\krel(\bs_0)}\right\}.
\end{equation}
On the event $\{|\widehat q^{\,\mathrm W}_{1-\alpha}-q_m^\star|\le\delta\}$, the events $\{|z_0|\le\widehat q^{\,\mathrm W}_{1-\alpha}\}$ and $\{|z_0|\le q_m^\star\}$ can differ only if $|z_0|\in(q_m^\star-\delta,q_m^\star+\delta]$. The density of $|z_0|$ is at most $(2/\pi)^{1/2}/\kappa_{0,\min}$ under Assumption~\ref{supp:ass_target}. Hence
\begin{equation}
\label{eq:supp_symdiff_bound}
\left|
\Prob\{|z_0|\le\widehat q^{\,\mathrm W}_{1-\alpha}\}-\Prob(|z_0|\le q_m^\star)
\right|
\le
\frac{2\delta}{\kappa_{0,\min}}\left(\frac2\pi\right)^{1/2}
+\Prob\bigl(|\widehat q^{\,\mathrm W}_{1-\alpha}-q_m^\star|>\delta\bigr).
\end{equation}
Take $\delta=v_m^{1/3}$. Since $v_m\ge1/m$, this $\delta$ exceeds $4/(cm)$ for large $m$, and since $v_m\to0$ it falls below $\delta_0$. Lemma~\ref{supp:lem_quant} bounds the last term by $2v_m^{1/3}/c^2$, so the right side of \eqref{eq:supp_symdiff_bound} is $O(v_m^{1/3})$. Combining with \eqref{eq:supp_leading} proves \eqref{eq:loc}.
\end{proof}

$\krel(\bs_0)$ in Theorem~\ref{thm:loc} characterizes the expected coverage. If it exceeds one, the transformed target error is large relative to the threshold, and that can lead to undercoverage. Similarly, for a value below one, it predicts conservative coverage. In the next subsection, we turn this characterization into a prospective diagnostic.

\subsection{Prospective implementation of the diagnostic}
\label{supp:impl_prosp}
Here, we develop a way to estimate $\krel$ from the data. 
The denominator of $\krel$ is a function of the calibration set, and is already computed as a part of the conformal inference. The numerator denotes the scale of the transformed target error. To estimate it, we hold out a set $\mathcal G$ of diagnostic observations from the target region. We do not use them for fitting, calibration or evaluation. Let $\mathcal D=\mathcal F\cup\mathcal C$. We transform each $j\in\mathcal G$ separately as a target given $\mathcal D$:
\begin{equation}
\label{eq:supp_diag_error}
e_j^{\mathrm{diag}}
=
\frac{
r_j-\bSigma_{j,\mathcal D}^{\mathrm w}
(\bSigma_{\mathcal D,\mathcal D}^{\mathrm w})^{-1}
\br_{\mathcal D}}
{
\left\{
\bSigma_{j,j}^{\mathrm w}
-
\bSigma_{j,\mathcal D}^{\mathrm w}
(\bSigma_{\mathcal D,\mathcal D}^{\mathrm w})^{-1}
\bSigma_{\mathcal D,j}^{\mathrm w}
\right\}^{1/2}}.
\end{equation}
No diagnostic observation conditions on another, so each $e_j^{\mathrm{diag}}$ uses the same information as the error at a future target. With $n_{\mathrm{diag}}=|\mathcal G|$, we estimate $\krel$ as
\begin{equation}
\label{eq:supp_krel_hat}
\widehat\krel
=
\frac{\xi_{1-\alpha}
\left\{n_{\mathrm{diag}}^{-1}\sum_{j\in\mathcal G}(e_j^{\mathrm{diag}})^2\right\}^{1/2}}
{\widehat q^{\,\mathrm W}_{1-\alpha}},
\qquad
\widehat p_{\mathrm{cov}}
=
G\!\left(\frac{\xi_{1-\alpha}}{\widehat\krel}\right).
\end{equation}

Suppose, in addition to the assumptions of Theorem~\ref{thm:loc}, that the diagnostic and target errors share the same scale and the correlation matrix $\bm R_{\mathcal G}$ of the diagnostic errors satisfies $\|\bm R_{\mathcal G}\|_F^2/n_{\mathrm{diag}}^2\to0$. Then $n_{\mathrm{diag}}^{-1}\sum_{j\in\mathcal G}(e_j^{\mathrm{diag}})^2/\kappa(\bs_0)^2\to1$ in probability, which together with Lemma~\ref{supp:lem_quant} and Theorem~\ref{thm:loc} shows that $\widehat p_{\mathrm{cov}}$ is consistent for the coverage. If the scales vary across the region, $\widehat\krel$ uses their root-mean-square, and $\widehat p_{\mathrm{cov}}$ is a regional approximation to the coverage averaged over the region.

\section{Simulation designs and implementation choices}
\label{supp:sec_impl}

\subsection{Data-generating mechanisms}
\label{supp:designs}

Section~\ref{sec:sim_setup} gives the general data-generating mechanism, including the mean function, the covariance family \eqref{eq:matern} and the random-sampling designs of Regime A. Below we specify the nonrandom designs of Regimes B and C. Except under the varying range design, all of them use the exponential correlation with $\phi=2$.

\subsubsection{Regime B: predictive information in calibration}

\begin{description}\itemsep2pt
\item[Halo $k$, $k\in\{10,20,40\}$.] For each target in turn, its $k$ nearest locations enter the calibration set. In case of overlapping neighbourhoods, we count them once, and exclude any target that falls in this neighbourhood. If this neighbourhood set exceeds $166$ locations, we keep the first $166$ in that order. Otherwise, we randomly choose from the remaining locations to fill it to $166$.
\item[Transect.] The $330$ fitting locations have $s_1\sim\operatorname{Unif}(0,1)$ and $s_2$ cycling over $\{0.1,0.3,0.5,0.7,0.9\}$ with independent $N(0,0.012^2)$ jitter, which gives $66$ observations per line. The remaining $171$ locations have uniform $s_1$ and $s_2$ cycling over $\{0.2,0.4,0.6,0.8\}$ with independent $N(0,0.02^2)$ jitter. Five of them are targets and the rest form the calibration set.
\item[Boundary.] The $330$ fitting locations are uniform on $[0.12,0.88]^2$. The remaining $171$ are uniform on the rest of $[0,1]^2$. Five of them are targets and the rest form the calibration set.
\item[Monitoring network.] The $330$ fitting locations are split equally among discs of radius $0.07$ centred at $(0.22,0.22)$, $(0.78,0.25)$ and $(0.5,0.8)$. The remaining $171$ locations come from a uniform distribution on $[0,1]^2$, and maintain at least $0.18$ distance from every centre. The first five of these are targets and the rest form the calibration set.
\end{description}

\subsubsection{Regime C: calibration--target comparability}
Locations are uniform on $[0,1]^2$ under quadrant, clustered calibration, block holdout and varying range.

\begin{description}\itemsep2pt
\item[Disjoint halves.] The $330$ fitting locations are uniform on $[0,0.45]\times[0,1]$. The remaining $171$ are uniform on $[0.55,1]\times[0,1]$, and five of them are targets.
\item[Quadrant.] The five targets are drawn from $\{s_1>0.5,s_2>0.5\}$. The other locations in this quadrant enter calibration, which is filled to $166$ from outside the quadrant.
\item[Clustered calibration.] The calibration set contains the $166$ locations nearest $(0.3,0.3)$. The five targets are drawn from the $120$ locations farthest from that point.
\item[Distant patch.] The $166$ calibration locations are uniform on the disc of radius $0.10$ centred at $(0.13,0.13)$. The other $335$ are uniform on $[0.42,1]^2$, and five of them are targets.
\item[Block holdout.] The five targets are drawn from $\{s_1>0.5,s_2>0.5\}$, and all other locations in this quadrant are removed. Of the locations outside it, $166$ chosen at random form the calibration set and the rest, about $210$, form the fitting set.
\item[Varying range.] The five targets are drawn from $\{s_1\ge0.5\}$, and the rest are split at random. The residual process consists of two independent exponential fields on either side of $s_1=0.5$, with $\phi=1$ on the left and $\phi=10$ on the right.
\end{description}

\subsection{Implementation choices and sensitivity}
\label{supp:impl}

\subsubsection{Software and covariance estimation}
\label{supp:impl_settings}

All methods fit the mean with \texttt{ranger} using $500$ trees and default settings. For whitening, the parameters of the exponential working covariance are estimated with \texttt{BRISC\_estimation}, using an intercept-only mean and the package defaults, which use $15$ nearest neighbours and locations ordered by the sum of their coordinates. GSCP and LSCP use the covariance-estimation routine supplied with \texttt{scp}.

\subsubsection{Choice of LSCP bandwidth}
\label{supp:impl_comp}

LSCP uses a kernel with bandwidth $h$ to weight each calibration observation based on its distance to the target. If $h$ is too small, the calibration observations carry too little weight for the weighted quantile to be finite, and the interval is unbounded. If $h$ is too large, LSCP behaves like GSCP. \citet{mao2024} navigate this tradeoff by choosing $h$ to minimise the average interval score over validation observations. In the split version only the calibration observations carry weight, so we validate on calibration observations and weight each one by the others. Their candidate grid gives no bounded interval at our sampling density, so we extend it to half the largest distance between calibration observations.

However, in Regime C the selected $h$ can be far too small. Since $h$ is selected using the calibration observations alone, the problem worsens the farther the calibration observations lie from the targets. Table~\ref{supp:tab_eta} shows that under block holdout, the selected $h$ produces unbounded intervals at $8\%$ of the targets. This rises to $76\%$ under clustered calibration and to every target under distant patch. An unbounded interval says nothing about the target, and counting it as covered inflates the coverage, to $0.986$ under clustered calibration. To bypass this problem, we use a fixed $h=0.40$ throughout the simulation. This produces similar results to the selected $h$ on Regimes A and B in terms of coverage and width. In Regime C, this choice of $h$ produces bounded intervals in all but the distant patch design, where it is unbounded at only $25\%$ of the targets. We can in theory reduce this by choosing an even larger bandwidth, but a larger $h$ increasingly moves LSCP toward GSCP on every other design, so we keep $h=0.40$.

\begin{table}[t]
\centering
\small
\setlength{\tabcolsep}{3pt}
\caption{LSCP with fixed bandwidth $h=0.40$ and selected bandwidth. Coverage/mean width are computed at nominal $0.90$. Unbounded intervals count as covering their target and are excluded from the mean width. $h=0.40$ is unbounded only under distant patch, at $24.8\%$ of targets, where its entry is computed on the bounded intervals.}
\label{supp:tab_eta}
\begin{tabular}{lcccc}
\toprule
Design & Fixed $h=0.40$ & Selected $h$ & Selected unbounded (\%) & Median selected $h$\\
\midrule
Exponential, $\phi=0.5$ & $.909/1.80$ & $.908/1.80$ & $0.0$ & $0.50$\\
Exponential, $\phi=1$   & $.918/1.90$ & $.915/1.90$ & $0.0$ & $0.52$\\
Exponential, $\phi=2$   & $.922/2.07$ & $.922/2.08$ & $0.0$ & $0.50$\\
Exponential, $\phi=6$   & $.911/2.50$ & $.908/2.50$ & $0.0$ & $0.48$\\
Exponential, $\phi=8$   & $.910/2.62$ & $.914/2.64$ & $0.0$ & $0.48$\\
Exponential, $\phi=30$  & $.911/3.47$ & $.910/3.47$ & $0.0$ & $0.50$\\
\Matern{}, $\nu=3/2$    & $.919/2.08$ & $.922/2.09$ & $0.0$ & $0.49$\\
\Matern{}, $\nu=5/2$    & $.909/1.99$ & $.904/1.98$ & $0.0$ & $0.51$\\
Gaussian                & $.920/2.03$ & $.917/2.03$ & $0.0$ & $0.52$\\
\midrule
Halo $10$               & $.907/2.30$ & $.907/2.30$ & $0.0$ & $0.50$\\
Halo $20$               & $.910/2.41$ & $.913/2.40$ & $0.0$ & $0.44$\\
Halo $40$               & $.916/2.62$ & $.925/2.64$ & $0.0$ & $0.34$\\
Transect                & $.912/2.30$ & $.917/2.32$ & $0.0$ & $0.36$\\
Boundary                & $.910/2.57$ & $.914/2.60$ & $0.0$ & $0.38$\\
Monitoring network       & $.917/3.14$ & $.921/3.14$ & $0.0$ & $0.28$\\
\midrule
Disjoint halves         & $.910/3.24$ & $.928/3.28$ & $0.0$ & $0.18$\\
Quadrant                & $.902/2.79$ & $.902/2.79$ & $0.0$ & $0.38$\\
Clustered calibration   & $.923/2.18$ & $.986/2.46$ & $75.8$ & $0.17$\\
Distant patch           & $.876/2.14$ & unbounded   & $100.0$ & $0.04$\\
Block holdout           & $.880/2.89$ & $.899/2.91$ & $7.9$ & $0.44$\\
Varying range           & $.902/2.59$ & $.906/2.64$ & $0.0$ & $0.36$\\
\bottomrule
\end{tabular}
\end{table}

\subsubsection{Factorization and ordering}
\label{supp:impl_ours}
Under exact whitening, the coverage and the distribution of the interval width do not depend on the ordering of the calibration locations or the factor used for whitening. In practice the covariance is estimated and the factor may be approximate, as with NNGP, so the interval may depend on both. The main results use the order supplied by the split, which does not depend on the responses. We compare it with a fresh random permutation and with two structured orderings, by coordinate sum and by maximin distance. We also replace the dense Cholesky factor with the NNGP factor with $K=25$ under the split order. Table~\ref{supp:tab_ours} shows that all five configurations give nearly identical coverage and width.

\begin{table}[t]
\centering
\small
\caption{Effect of factorization and ordering in whitening. Coverage and mean width are computed at nominal $0.90$, and averaged over the nine Regime A designs.}
\label{supp:tab_ours}
\begin{tabular}{llcc}
\toprule
Factor & Ordering & Coverage & Mean width\\
\midrule
Dense Cholesky & as supplied by the split & $0.904$ & $2.113$\\
Dense Cholesky & coordinate sum           & $0.906$ & $2.115$\\
Dense Cholesky & maximin                  & $0.904$ & $2.113$\\
Dense Cholesky & random                   & $0.905$ & $2.113$\\
NNGP, $K=25$   & as supplied by the split & $0.905$ & $2.116$\\
\bottomrule
\end{tabular}
\end{table}

\subsubsection{Departure from Gaussianity}
\label{supp:impl_ngauss}

We replace the Gaussian field by the multivariate $t$ field $r=W^{1/2}z$, where $z$ is a Gaussian field with covariance \eqref{eq:matern} and $W\sim\operatorname{IG}(\nu_t/2,\nu_t/2)$ is independent of $z$. We take $\nu_t\in\{5,10\}$, with the exponential correlation at $\phi=2$ and the Gaussian correlation at $\phi=6$. Table~\ref{supp:tab_ngauss} shows that all methods retain near-nominal coverage and that whitening gives the narrowest interval in every setting.

\begin{table}[t]
\centering
\small
\caption{Coverage/mean width at nominal coverage $0.90$ under a spatial $t$ process.}
\label{supp:tab_ngauss}
\begin{tabular}{llcccc}
\toprule
Correlation & $\nu_t$ & Naive & GSCP & LSCP & Whitened\\
\midrule
Exponential, $\phi=2$ & $10$ & $.899/3.27$ & $.903/2.13$ & $.909/2.19$ & $.903/\mathbf{2.04}$\\
Exponential, $\phi=2$ & $5$  & $.906/3.45$ & $.906/2.26$ & $.915/2.34$ & $.912/\mathbf{2.18}$\\
Gaussian, $\phi=6$    & $10$ & $.906/3.95$ & $.901/2.10$ & $.909/2.18$ & $.904/\mathbf{2.02}$\\
Gaussian, $\phi=6$    & $5$  & $.902/4.17$ & $.909/2.20$ & $.914/2.26$ & $.906/\mathbf{2.12}$\\
\bottomrule
\end{tabular}
\end{table}

\section{Fine particulate matter monitoring network data processing}
\label{supp:data}

\paragraph{Data.}
We use 2022 annual mean PM$_{2.5}$ concentrations from the Environmental Protection Agency's Air Quality System. We only consider monitors with at least $250$ daily observations and keep one monitor per site. We keep monitors in the contiguous United States, together with the three monitors in Mexico that the agency reports under state code $80$. 

\paragraph{Covariates.}
The agency's site listing provides elevation, land use and location setting for each monitor. Land use is a categorical variable that records the main use of the land around the monitor, such as commercial, residential, industrial, desert, forest or mobile. Location setting records whether the monitor is in an urban, suburban or rural area. We also include wind speed from TerraClimate (2022 annual mean) and population density from WorldPop (2020), both extracted at the monitor coordinates.

\paragraph{Coordinates.}
We cannot directly use the Euclidean distance between raw latitude and longitude, since the Earth's curvature makes it a distorted measure of distance on the ground. We therefore project the coordinates using the Albers equal-area conic projection for the contiguous United States. This provides planar coordinates in kilometres with minimal domain-wide distance distortion. We divide both coordinates by their largest range, $4547$ km, so that the monitors lie in the unit square. This puts the coordinates on the same scale as the simulations, and LSCP uses the same bandwidth $h=0.40$.
\end{document}